\documentclass{article}

\usepackage[utf8]{inputenc} 
\usepackage[T1]{fontenc}    

\usepackage[nonatbib,preprint]{ewrl_2023}
\usepackage[style=alphabetic,natbib=true]{biblatex}
\usepackage{url}            
\usepackage{booktabs}       
\usepackage{microtype}      
\usepackage{xcolor}         

\usepackage{url}            
\usepackage{booktabs}       
\usepackage{amsfonts,amssymb}       
\usepackage{mathtools, nccmath}
\usepackage{graphicx,wrapfig}
\usepackage{algorithmic,algorithm,float}
\usepackage{subcaption}
\usepackage{caption}
\usepackage{amsthm}
\usepackage{dsfont}
\usepackage{mathrsfs}
\usepackage{enumitem}
\usepackage{tcolorbox}

\newtheorem{proposition}{Proposition}
\theoremstyle{definition}
\newtheorem{remark}{Remark}

\usepackage[colorlinks=true,citecolor=blue,linkcolor=blue]{hyperref}

\usepackage{comment}
\excludecomment{extra}

\newcommand\ALPHABET{\mathcal}

\newcommand\EXPLORE{\pi_{\text{expl}}}
\newcommand\PR{\mathds{P}}
\newcommand\EXP{\mathds{E}}
\newcommand\IND{\mathds{1}}
\newcommand\reals{\mathds{R}}

\newcommand\red[1]{{\color{red}#1}}

\newtheorem{definition}{Definition}

\newtheorem{theorem}{Theorem}
\newtheorem{lemma}{Lemma}

\DeclareMathOperator\TV{TV}
\DeclareMathOperator\Was{Was}
\DeclareMathOperator\MMD{MMD}
\DeclareMathOperator\SPAN{span}
\DeclareMathOperator\LIP{Lip}
\newcommand*\F{\mathfrak{F}}

\title{Approximate information state based convergence analysis of recurrent Q-learning}
\author{Erfan Seyedsalehi \\
McGill University\\
\texttt{seyederfan.seyedsalehi@mail.mcgill.ca} 
\And
Nima Akbarzadeh \\
McGill University \\
\texttt{nima.akbarzadeh@mail.mcgill.ca} 
\And 
Amit Sinha \\
McGill University \\
\texttt{amit.sinha@mail.mcgill.ca}
\And 
Aditya Mahajan \\
McGill University \\
\texttt{aditya.mahajan@mcgill.ca} 
}

\begin{document}

\maketitle

\begin{abstract}
In spite of the large literature on reinforcement learning (RL) algorithms for partially observable Markov decision processes (POMDPs), a complete theoretical understanding is still lacking. In a partially observable setting, the history of data available to the agent increases over time so most practical algorithms either truncate the history to a finite window or compress it using a recurrent neural network leading to an agent state that is non-Markovian. In this paper, it is shown that in spite of the lack of the Markov property, recurrent Q-learning (RQL) converges in the tabular setting. Moreover, it is shown that the quality of the converged limit depends on the quality of the representation which is quantified in terms of what is known as an approximate information state (AIS). Based on this characterization of the approximation error, a variant of RQL with AIS losses is presented. This variant performs better than a strong baseline for RQL that does not use AIS losses. It is demonstrated that there is a strong correlation between the performance of RQL over time and the loss associated with the AIS representation.
\end{abstract}

\section{Introduction}

In recent years, Reinforcement Learning (RL) has witnessed many successes such as achieving human-level performance in Go \citep{silver2016mastering}, learning to play Atari \citep{mnih2013atari,Mnih2015atari}, as well as solving many control problems arising in engineering and robotics \citep{schulman2015trust,Schulman2017ppo,haarnoja2018sac,tassa2018deepmind}. These successes are achieved by algorithms with strong theoretical basis~\citep{SuttonBarto2018,BertsekasTsitsiklis1996}. However, RL theory, for the most part, is limited to models with full state information.

In various applications such as finance, healthcare, and robotics, the agent does not observe the full state of the environment. Such partially observed systems are mathematically modeled as partially observable Markov decision processes (POMDPs). When the system model is known, POMDPs can be viewed as MDPs by considering the belief state (i.e., the posterior distribution of the partially observed environmental state) as an information state~\citep{Astrom1965}. Furthermore, there are various efficient algorithms to compute approximately optimal planning solutions \citep{ShaniPineauKaplow2013}.

However, it is not possible to generalize these planning results to develop learning algorithms because constructing the belief state requires knowledge of the model. So, an agent operating in an unknown partially observed environment cannot construct a belief state based on its observation. Two approaches are commonly used in the literature to circumvent this conceptual difficulty (i)~use a finite window of observations (rather than the full history) and (ii)~use a recursively updateable (or recurrent) \emph{agent state}. A key difficulty in analyzing these learning algorithms is that the state of the agent may evolve in a non-Markovian manner. Furthermore, for the recurrent algorithm, the representation mapping histories to agent states needs to be learnt in parallel, which is especially difficult in sparse reward environments. So, even though there is a rich and large literature on RL theory for POMDPs~\citep[see][and follow-up liteature]{lin1992memory,Littman1994a,singh1994learning,jaakkola1994reinforcement, mccallum1996reinforcement,Whitehead1995,Loch1998,perkins2002existence}, much of the literature either analyzes the case where the agent does not have a memory, or only provides empirical evidence but does not include a detailed convergence or approximation analysis. 

In this paper, we investigate one of the most popular RL algorithms for POMDPs: Recurrent Q-learning (RQL), which uses a recurrent neural network (RNN) for approximating a history-based Q-function. RQL was initially proposed by \cite{schmidhuber1991reinforcement, harp1992recurrent} with a substantial follow up literature \citep{bakker2002reinforcement,Wierstra2007,Wierstra2010,daswani2013q}. There is growing empirical evidence suggesting that variants of RQL work well in practice \citep{hausknecht2015recqlearning,kapturowski2018recurrent,Foerster2016,OU2021300,mousavi2017learning,sorokin2015deep}. However, a detailed theoretical understanding of the algorithm is lacking. 

\paragraph{Review of theoretical papers analyzing variants of Q-learning for POMDPs}
There are a few recent papers which analyze closely related problems.  A general framework of approximation for POMDPs based on the notion of approximate information state (AIS) is proposed in~\citep{subramanian2022approximate}. It is shown that a planning policy computed using an AIS is approximately optimal with bounded loss of optimality. Furthermore, an actor-critic algorithm which uses the AIS-approximation losses as an auxiliary loss is presented and it is demonstrated that the proposed algorithm has good empirical performance. Even though AIS may be viewed as a recurrent agent state, the analysis presented in \cite{subramanian2022approximate} is for actor-critic algorithms and is not directly applicable to RQL. 

 Approximate planning and Q-learning for POMDPs with a finite window of observations is presented in~\cite{kara2022convergence, kara2022near}, where the approximation error and convergence are quantified. The special case of just using the current observation has also been analyzed in~\cite{singh1994learning}. Even though our analysis uses similar technical tools as~\cite{singh1994learning, kara2022convergence,kara2022near}, the analysis of these papers is for Q-learning with finite window of observations and is not directly applicable to~RQL. 

Regret guarantees for RL agents operating in non-Markovian environments and using an optimistic variant of Q-learning is presented in~\cite{dong2021simple}. Even though the agent state in~\cite{dong2021simple} is a recurrent state, 
the analysis of \cite{dong2021simple} is tuned for an optimistic variant of Q-learning and is not directly applicable to RQL. 

The papers closest to our work are \cite{majeed2018qlearning,chandak2022reinforcement} which establish convergence of Q-learning in a non-Markovian environment under the assumption that the state-observation-action process is stationary and ergodic (an additional technical assumption of state uniformity is also imposed in \cite{majeed2018qlearning}). Asymptotic rates of convergence are also characterized in \cite{chandak2022reinforcement}. However, \cite{majeed2018qlearning, chandak2022reinforcement} do not present explicit approximation bounds. In our analysis, we do not assume that the system is stationary and provide approximation bounds.

\paragraph{Contributions} Our main contributions are as follows. First, we show that in spite of the non-Markovian evolution of the agent state, RQL converges. As far as we are aware, this is the first result that establishes the convergence of RQL without making any assumptions on the stationarity of the agent state. Second, using ideas from approximate information state (AIS)~\cite{subramanian2022approximate}, we quantify the quality of the converged limit of RQL in terms of error in representation.  Third, we propose a variant of RQL called RQL-AIS which incorporates AIS losses. We illustrate via detailed numerical experiments that RQL-AIS learns better than R2D2~\cite{kapturowski2018recurrent}, which is the state-of-the-art RQL algorithm for POMDPs. 
We also empirically demonstrate that there is a strong correlation between the performance of RQL over time and the loss associated with the AIS representation.

\section{Background}

\paragraph{Partially observable Markov decision processes (POMDPs)}
A partially observable Markov decision process (POMDP) is a tuple $\langle \ALPHABET S, \ALPHABET Y, \ALPHABET A, P, O, r, \gamma \rangle$ where $\ALPHABET S$ denotes the state space, $\ALPHABET Y$ denotes the observation space, $\ALPHABET A$ denotes the action space, $P \colon \ALPHABET S \times \ALPHABET A \to \Delta(\ALPHABET S)$ denote the state transition matrix, $O \colon \ALPHABET S \times \ALPHABET A \to \Delta(\ALPHABET Y)$ denotes the observation probability matrix, $r \colon \ALPHABET S \times \ALPHABET A \to \reals$ is the reward function and $\gamma \in [0,1)$ denotes the discount factor.

We follow the standard notation from probability theory and use uppercase letters to denote random variables and lowercase letters to denote their realizations. In particular, we use $S_t$, $Y_t$, $A_t$ to denote the state, observation, and action at time~$t$ and $H_t = (Y_1, A_1, Y_2, A_2, \dots, Y_t)$ to denote the history of observations and actions until time~$t$. Let ${\ALPHABET H}_t = \ALPHABET Y^t \times {\ALPHABET A}^{t-1}$ denote the space of all histories until time $t$. We use $R_t = r(S_t, A_t)$ to denote the random reward received at time~$t$.

A policy $\pi = (\pi_1, \pi_2, \dots)$ is a collection of history dependent randomized decision rules $\pi_t \colon {\ALPHABET H}_t \to \Delta(\ALPHABET A)$ such that the action at time~$t$ is chosen according to  $A_t \sim \pi_t(H_t)$. 
The performance of any policy $\pi$ starting from history $h_t \in \ALPHABET H_t$ at time~$t$ is given by the value function $V^{\pi}_t(h_t)$ defined as
\begin{align} \label{eq:value_fn_defn}
 V^{\pi}_t(h_t) \coloneqq \EXP^{\pi} \Bigl[ \medop\sum_{\tau=t}^{\infty} \gamma^{\tau-t} r(S_{\tau}, A_{\tau}) \Bigm| H_t = h_t \Bigr].
\end{align}
The corresponding action-value function or Q-function $Q^{\pi}_t(h_t, a_t)$ is defined as
\begin{align} \label{eq:Q_value_fn_defn}
 Q^{\pi}_t(h_t, a_t) \coloneqq \EXP^{\pi} \bigl[ r(S_t, A_t) + \gamma V^{\pi}_{t+1}(H_{t+1}) \bigm| H_t = h_t, A_t = a_t \bigr].
\end{align}

    A policy $\pi^\star$ is called \emph{optimal} if for every other policy $\pi$, we have
        \(
           V^{\pi^\star}_t(h_t) \ge V^{\pi}_t(h_t), 
        \)
        for all $t \in \mathds{Z}_{> 0}$ and $h_t \in \ALPHABET H_t$.
    The value function and action-value function of optimal policies are denoted by $V^\star_t$ and $Q^\star_t$.

\paragraph{Integral Probability Metrics (IPMs)} \label{subsec:IPM}
Integral probability metrics (IPMs) are a family of semi-metrics on probability measures defined in terms of a dual relationship~\citep{muller1997integral}.

\begin{definition}\label{def:ipm}
  Let $(\ALPHABET X, \mathscr{G})$ be a measurable space and $\F$
  be a class of measurable real-valued functions on $(\ALPHABET X,
  \mathscr{G})$. The integral probability metric (IPM) between two probability distributions $\mu, \nu \in \mathscr{P}(\ALPHABET X)$ with respect to the function class $\F$ is defined as
  \(
    d_{\F}(\mu, \nu) \coloneqq \sup_{f \in \F}\bigl |
    \int_{\ALPHABET{X}} fd\mu - \int_{\ALPHABET{X}} fd\nu \bigr|.
  \)
\end{definition}

A key property of IPMs is that for any function $f$ (not necessarily in $\F$), we have
\begin{equation} \label{eq:minkowski}
  \Bigl| \medint\int_{\ALPHABET X} f d\mu - \medint\int_{\ALPHABET X} f d\nu \Bigr|
  \le \rho_{\F}(f) \cdot d_{\F}(\mu, \nu),
\end{equation}
where 
\(
  \rho_{\F}(f) \coloneqq \inf\{
  \rho \in \reals_{> 0} : \rho^{-1} f \in \F \}
\)
is called the Minkowski functional of $f$.

Some examples of IPMs are as follows:
(i)~\textbf{Total variation distance} where $\F = \F_{\TV} \coloneqq \{f : \SPAN(f) \le 1\}$ (where $\SPAN(f)$ is the span semi-norm of a function). For this case, $\rho_{\TV}(f) = \SPAN(f)$.
(ii)~\textbf{Wasserstein distance} where $\F = \F_{\Was} \coloneqq \{ f : \LIP(f) \le 1 \}$ (where $\ALPHABET X$ is a metric space and $\LIP(f)$ is the Lipschitz constant of the function $f$, computed with respect to the metric on $\ALPHABET X$). For this case, $\rho_{\Was}(f) = \LIP(f)$.
(iii)~\textbf{Maximum mean discrepancy (MMD)} where $\F = \F_{\MMD} \coloneqq \{ f \in \mathcal{H} \colon \| f \|_{\mathcal{H}} \leq 1\}$ (where $\mathcal{H}$ is a reproducing kernel Hilbert space of real-valued functions on $\ALPHABET X$ and $\| f \|_{\mathcal{H}}$ is the Hilbert space norm of $f$). For this case, $\rho_{\MMD}(f) = \|f\|_{\mathcal{H}}$. 

\paragraph{Approximate information state (AIS)}

Approximate information state (AIS) is a self-predictive representation for POMDPs, first proposed in \cite{subramanian2022approximate}.

\begin{definition} \label{def:AIS}
Given a function class $\F$ and a measurable space $\ALPHABET Z$, an $(\varepsilon_t, \delta_t)_{t \ge 1}$ AIS-generator is a tuple $\langle \{\sigma_t\}_{t \ge 1}, \tilde P, \tilde r\rangle$ of history compression functions $\sigma_t \colon \ALPHABET H_t \to \ALPHABET Z$, transition approximator $\tilde P \colon \ALPHABET Z \times \ALPHABET A \to \Delta(\ALPHABET Z)$, and reward approximator $\tilde r \colon \ALPHABET Z \times \ALPHABET A \to \reals$ such that
for all $h_t \in \ALPHABET H_t$ and $a_t \in \ALPHABET A$
\begin{align*}
\left|\mathbb{E}\left[R_t \mid H_t=h_t, A_t=a_t\right]-\tilde{r}\left(\sigma_t\left(h_t\right), a_t\right)\right|
&\le \varepsilon_t, 
\\
 d_{\F}\bigl(\mathbb{P}\left(Z = \cdot \mid H_t=h_t, A_t=a_t\right), \tilde{P}(\cdot \mid \sigma_t\left(h_t\right), a_t)\bigr) & \le \delta_t.
\end{align*}
\end{definition}
Given an AIS generator, consider the following dynamic program:
\begin{equation}\label{eq:ADP}
    \tilde Q(z,a) = \tilde r(z,a) + \gamma \medint\int_{\ALPHABET Z} \tilde P(dz' \mid z,a) \max_{\tilde a \in \ALPHABET A}
    \tilde Q(z', \tilde a).
\end{equation}
Let $\tilde Q^\star$ denote the unique fixed point of~\eqref{eq:ADP}.
Define $\tilde V^\star \colon \ALPHABET Z \to \reals$ to be the value function corresponding to $\tilde Q^\star$ and $\tilde \pi^\star \colon \ALPHABET Z \to \ALPHABET A$ to be the greedy policy\footnote{\label{fnt:arg-max}To avoid ambiguity due to the non-uniqueness of the arg-max, we assume that there is a deterministic rule to break ties is pre-specified so that the arg-max is always unique.} with respect to $\tilde Q^\star$, i.e.,
\[
  \tilde V^\star(z) = \max_{a \in \ALPHABET A} \tilde Q^\star(z,a), 
  \quad\text{and}\quad
  \tilde \pi^\star(z) = \arg\max_{a \in \ALPHABET A} \tilde Q^\star(z,a).
\]
Then, the following result is a generalization of~\cite[Theorem 27]{subramanian2022approximate}.
\begin{theorem}\label{thm:ais}
    Let $\tilde \pi = (\tilde \pi_1, \tilde \pi_2, \dots)$ be a time-varying and history-dependent policy  given by $\tilde \pi_t(h_t) = \tilde \pi^\star(\sigma_t(h_t))$. Then, for any time~$t$ and any history $h_t \in \ALPHABET H_t$ and action $a_t \in \ALPHABET A$, we have
    \begin{itemize}[leftmargin=1em, topsep=0pt, partopsep=0pt]
    \item \textbf{Bounds on value approximation:}
    \begin{align}
        \bigl| Q^\star_t(h_t, a_t) - \tilde Q^\star(\sigma_t(h_t), a_t) \bigr| &\le 
        (1-\gamma)^{-1}\bigl[ \bar \varepsilon_t + \gamma \bar \delta_t \rho_{\F}(\tilde V^\star) \bigr],
        \\
        \bigl| V^\star_t(h_t) - \tilde V^\star(\sigma_t(h_t)) \bigr| &\le 
        (1-\gamma)^{-1}\bigl[ \bar \varepsilon_t + \gamma \bar \delta_t \rho_{\F}(\tilde V^\star) \bigr],
    \end{align}
    where
    \(
        \bar \varepsilon_t = (1-\gamma) \sum_{\tau = t}^\infty \gamma^{\tau -t }\varepsilon_{\tau}
    \)
    and
    \(
        \bar \delta_t = (1-\gamma) \sum_{\tau = t}^\infty \gamma^{\tau -t }\delta_{\tau}.
    \)
    \item \textbf{Bounds on policy approximation:}
    \begin{equation} 
        \bigl| V^\star_t(h_t) - V^{\tilde \pi}_t(h_t) \bigr| \le 
        2 (1-\gamma)^{-1}\bigl[  \bar \varepsilon_t +  \gamma \bar \delta_t \rho_{\F}(\tilde V^\star) \bigr].
    \end{equation}
    \end{itemize}
    
\end{theorem}

\paragraph{Recurrent Neural Networks (RNN)}

Recurrent neural networks (RNNs) are neural networks with feedback connections that are used to process sequential data by keeping track of a state. At an abstract level, we may model an RNN with a hidden state\footnote{Normally, the hidden state of an RNN is denoted using $h_t$. However, we are using $h_t$ to denote the history of a POMDP. So, we use $z_t$ to denote the hidden state of an RNN.} $z_t$ to be a function of the past sequence of inputs $x_1, \dots, x_t$, which is updated recursively using a non-linear activation function: $z_t = f(z_{t-1}, x_t)$. Typically, $f(\cdot)$ is a parameterized family of functions, e.g., in vanilla RNN, $f(z_{t-1}, x_t) = \tanh(W_{zz} z_{t-1} + W_{zx} x_t + W_b)$, where $(W_{zz}, W_{zx}, W_b)$ are parameters. In practice, one uses more sophisticated RNN architectures such as long short-term memory (LSTM)~\citep{lstm} or gated recurrent units (GRUs) \citep{gru}, which avoid the problem of vanishing gradients.

\paragraph{Recurrent Q-learning}

Recurrent Q-learning (RQL) is a variant of Q-learning algorithm for POMDPs, which uses an RNN to estimate the Q-function $Q_t(h_t, a_t)$ \citep{schmidhuber1991reinforcement,harp1992recurrent}. In particular, an RNN with input $(Y_t, A_{t-1})$ is used to generate a hidden state $z_t \in \ALPHABET Z$ which is updated recursively as $z_t = f(z_{t-1}, y_t, a_{t-1})$, where $f(\cdot)$ is the update function of an RNN. We will sometimes write $z_t = \sigma_t(h_t)$ to highlight the fact that $z_t$ is a function of the history $h_t$.

In RQL the learning agent uses an exploration policy $\EXPLORE$ to generate experience and updates an estimate of the Q-function using the following recursion:
\begin{equation}\label{eq:Q-update}
    \widehat Q_{t+1}(z_t, a_t) = \widehat Q_t(z_t, a_t) + \alpha_t(z_t, a_t)\bigl[ R_t + \gamma \max_{\tilde a \in \ALPHABET A} \widehat Q_t(z_{t+1}, \tilde a) - \widehat Q_t(z_t, a_t) \bigr]
\end{equation}
where $\{\alpha_{t}(z_t, a_t)\}_{t \geq 1}$ is the learning rate. Define $\hat V_t \colon \ALPHABET Z \to \reals$ to be the value function corresponding to $\widehat Q_t$ and $\hat \pi_t \colon \ALPHABET Z \to \ALPHABET A$ to be the greedy policy w.r.t.\ $\widehat Q_t$, i.e., 
\[
    \hat V_t(z) = \max_{a \in \ALPHABET A}\widehat Q_t(z,a)
    \quad\text{and}\quad
    \hat \pi_t(z) = \arg \max_{a \in \ALPHABET A}\widehat Q_t(z,a).
\]

\section{Theoretical results}

The key challenge in characterizing the convergence of RQL is that the agent state $\{Z_t\}_{t \ge 1}$ is not a controlled Markov process. Therefore, the standard results on the convergence of Q-learning~\cite{jaakkola1993convergence} are not directly applicable. In Sec.~\ref{sec:convergence}, we show that it is possible to adapt the standard convergence arguments to show that RQL converges. 

The quality of the converged solution depends on choice of the exploration policy as well as the representation. The dependence of the representation is not surprising. For example, it is clear that when the representation is bad (e.g., a representation that maps all histories to a single agent state), then RQL will converge to a limit which is far from optimal. So, it is important to quantify the degree of sub-optimality of the converged limit. We do so in Sec.~\ref{sec:approx}.

\begin{extra}
In this section, we establish the theoretical properties of RQL. In Sec.~\ref{sec:convergence}, we assume that the state representation is fixed and given by some recurrent network of the form~\eqref{eq:Z-update} and the Q-function is updated according to~\eqref{eq:Q-update}. We show that under mild technical conditions, the iterates of~\eqref{eq:Q-update} converge almost surely to a limit.
\end{extra}

\subsection{Establishing the convergence of RQL}\label{sec:convergence}
\begin{lemma}\label{lem:markov}
    Under any policy $\pi \colon \ALPHABET Z \to \Delta(\ALPHABET A)$, the process $\{(S_t, Y_t, Z_t, A_t)\}_{t \ge 1}$ is a Markov chain.
\end{lemma}
\begin{extra}
\begin{lemma}\label{lem:markov}
    Under any policy $\pi \colon \ALPHABET Z \to \Delta(\ALPHABET A)$, the process $\{S_t, Y_t, Z_t \}_{t \ge 1}$ is a controlled Markov process controlled by $\{A_t\}_{t \ge 1}$. In particular, for any time~$t$,
    \begin{equation}\label{eq:markov}
        \PR(S_{t+1}, Y_{t+1}, Z_{t+1} \mid S_{1:t}, Y_{1:t}, Z_{1:t}, A_{1:t})
        =
        \PR(S_{t+1}, Y_{t+1}, Z_{t+1} \mid S_{t}, Y_{t}, Z_{t}, A_{t}).
    \end{equation}
    Therefore, the process $\{(S_t, Y_t, Z_t, A_t)\}_{t \ge 1}$ is a Markov chain.
\end{lemma}
\begin{proof}
        We have
        \begin{align*}
            \PR(S_{t+1}, Y_{t+1}, Z_{t+1} | S_{1:t}, Y_{1:t}, Z_{1:t}, A_{1:t}) 
            &= \IND_{\{ Z_{t+1} = f(Z_{t}, Y_{t+1}, A_{t})\}} O(Y_{t+1} | S_{t+1}, A_{t}) P(S_{t+1} | S_{t}, A_{t}) \\
            &= \PR(S_{t+1}, Y_{t+1}, Z_{t+1} | S_{t}, Y_{t}, Z_{t}, A_{t}),
        \end{align*}
       which establishes~\eqref{eq:markov}. Since $A_t \sim \pi(Z_t)$, we have that  $\{(S_t, Y_t, Z_t, A_t)\}_{t \ge 1}$ is a Markov chain.
\end{proof}
\end{extra}
We impose the following assumptions:
\begin{enumerate}[leftmargin=2.5em,topsep=0pt,partopsep=0pt,noitemsep]
    \item[\textbf{(A1)}]
        The state space, action space, and the recurrent state space are finite. 
    \item[\textbf{(A2)}]
        The exploration policy $\EXPLORE \colon \ALPHABET Z \to \Delta(\ALPHABET A)$ is such that the Markov chain $\{(S_t, Y_t, Z_t, A_t)\}_{t \ge 1}$ has a unique stationary distribution $\xi$. Moreover, for every $(s,y,z,a)$, $\xi(s,y,z,a) > 0$. 
    \item[\textbf{(A3)}]
        The learning rate $\alpha_t(z,a)$ is 
        given by 
        \(
           \alpha_t(z,a) =  \IND_{\{Z_t = z, A_t = a \}} \big/ 
           \bigl( {1 + \textstyle\sum_{\tau = 0}^t \IND_{\{Z_\tau = z , A_\tau = a \}}} \bigr).
        \)
\end{enumerate}

We impose assumption \textbf{(A1)} to analyze the simplest version of the RQL. Assumption \textbf{(A2)} is a mild assumption on the exploration policy and is commonly assumed in several variations of Q-learning with function approximation~\citep{Tsitsiklis1997temporal}. Assumption \textbf{(A3)} is a common assumption on the step-size of stochastic approximation algorithms.

For the ease of notation, we continue to use $\xi$ to denote marginal and conditional distributions w.r.t.\ $\xi$. For example, $\xi(y,z,a) = \sum_{s \in \ALPHABET S} \xi(s,y,z,a)$ and similar notation holds for other marginals. Similarly, $\xi(s | z) = \xi(s,z)/\xi(z)$ and similar notation holds for other conditional distributions. 

    Given a steady-state distribution $\xi$ corresponding to the exploration policy, define
    a reward function $r_\xi \colon \ALPHABET Z \times \ALPHABET A \to \reals$ and transition probability $P_\xi \colon \ALPHABET Z \times \ALPHABET A \to \Delta(\ALPHABET Z)$ as follows:
    \begin{align*}
        r_\xi(z,a) &= \medop\sum_{s \in \ALPHABET S} r(s,a) \xi(s \mid z, a), \\
        P_\xi(z' \mid z, a) &= \medop\sum_{s \in \ALPHABET S}  \xi(s \mid z, a) \medop\sum_{s' \in \ALPHABET S} P(s' | s, a) \medop\sum_{y' \in \ALPHABET Y} O(y' | s', a) \IND_{\{ z' = f(z,y',a) \}}.
    \end{align*}
    Furthermore, define $Q_\xi^\star$ to be the unique fixed point of the following fixed point equation:
    \begin{equation}\label{eq:Q-xi}
    Q_\xi^\star(z, a) = r_\xi (z, a) + \gamma \medop\sum_{z' \in \ALPHABET Z} P_\xi(z' \mid z, a) \max_{\tilde a \in \ALPHABET A} Q_\xi^\star(z, \tilde a). 
    \end{equation}
Define $V^\star_{\xi} \colon \ALPHABET Z \to \reals$ be the value function corresponding to $Q^\star_{\xi}$ and $\pi^\star_\xi \colon \ALPHABET Z \to \ALPHABET A$ to be the greedy policy with respect to $Q^\star_\xi$, i.e.,
\[
  V^\star_\xi(z) = \max_{a \in \ALPHABET A} Q^\star_\xi(z,a), 
  \quad\text{and}\quad
  \pi^\star_\xi(z) = \arg\max_{a \in \ALPHABET A} Q^\star_\xi(z,a).
\]

\begin{theorem}\label{thm:convergence}
    Under Assumptions~{\bf (A1)}--{\bf (A3)}, the iterates $\{\widehat Q_t\}_{t \ge 1}$ of~\eqref{eq:Q-update} converge almost surely to $Q_\xi^\star$ given by~\eqref{eq:Q-xi}. Therefore $\{\hat \pi_t\}_{t \ge 1}$ converges to $\pi^\star_\xi$ (see footnote~\ref{fnt:arg-max} for uniqueness of arg-max).
\end{theorem}
\begin{proof}[Proof outline]
   The main idea of the proof is inspired from \cite{jaakkola1993convergence, kara2022convergence}.
   To establish that $\widehat Q_t \to Q^\star_\xi$, a.s., we will show that $\Delta_t \coloneqq \widehat Q_t - Q^\star_\xi \to 0$, a.s. Define $\hat V_t(z) = \max_{a \in \ALPHABET A} \widehat Q_t(z,a)$. Combining~\eqref{eq:Q-update} and~\eqref{eq:Q-xi}, we get that
   \begin{equation}\label{eq:Delta}
       \Delta_{t+1}(z,a) = (1 - \alpha_t(z,a)) \Delta_t(z,a) + \alpha_t(z,a) \bigl[ F^1_t(z,a) +  F^2_t(z,a) \bigr],
   \end{equation}
   where
   \allowdisplaybreaks
   \begin{align}
      F^1_t(z,a) &= \gamma \hat V_t(Z_{t+1}) - \gamma V^\star_\xi(Z_{t+1}),  
      \label{eq:F1}
      \\
      F^2_t(z,a) &=  \textstyle R_t - r_\xi(z,a) + \gamma V^\star_\xi(Z_{t+1}) - \gamma \sum_{z' \in \ALPHABET Z} P_\xi(z' | z, a) V^\star_\xi(z') .
      \label{eq:F2}
   \end{align}
   Following \cite{jaakkola1993convergence}, we view~\eqref{eq:Delta} as a linear system with two inputs and do ``state splitting'' to write
    \begin{equation}\label{eq:Delta-sum}
        \Delta_t(z,a) = W^1_t(z,a) + W^2_t(z,a) 
    \end{equation}
    where for $i \in \{1,2\}$, each ``state component'' $W^i_1(z,a)$ is initialized to $0$ and evolves for $t \ge 1$ as
    \[
        W^i_{t+1}(z,a) = (1 - \alpha_t(z,a)) W^i_t(z,a) + \alpha_t(z,a) F^i_t(z,a),
        \quad i \in \{1, 2\}.
    \]
     From~\eqref{eq:F1}, we have that
     \begin{equation}\label{eq:F0-bound}
       F^1_t(z,a) 
       = \gamma \bigl[ \hat V_{t}(Z_{t+1}) - V^\star_\xi(Z_{t+1}) \bigr]
       \le \gamma \| \hat V_{t} - V^\star_\xi \|_{\infty} 
       \le \gamma \| \widehat Q_t - Q^\star_\xi \|_{\infty}
       = \gamma \| \Delta_t \|_\infty.
     \end{equation}
    Using assumptions \textbf{(A1)}--\textbf{(A3)}, we can show that $W^2_t(z,a) \to 0$, a.s., for all $(z,a)$. See the supplementary material for proof. Therefore, there exists a set $\Omega_0$ such that $\PR(\Omega_0) = 1$ and for every $\omega \in \Omega_0$ and any $\epsilon > 0$, there exists a $T(\omega, \epsilon)$ such that for all $t > T(\omega, \epsilon)$, $| W^2_t(z,a) | < \epsilon$, a.s., for all $(z,a)$.
     Now, pick $C$ such that $\gamma(1 + 1/C) < 1$.  For any $t > T(\omega, \epsilon)$ if $\|W^1_t(z,a)\|_{\infty} > C \epsilon$, then
     \begin{equation}\label{eq:F1-bound}
        F^1_t(z,a) \le \gamma \| \Delta_t \|_{\infty} \le \gamma \| W^1_t \|_\infty + \gamma \epsilon 
        < \gamma\bigl(1 + \tfrac{1}{C}\bigr) \| W^1_t\|_\infty 
        < \| W^1_t\|_{\infty},
    \end{equation}
    where the first inequality uses~\eqref{eq:Delta-sum}, the second uses the triangle inequality, and the others follow from the definition of $C$ and $\epsilon$.  Consequently, for any $t > T(\omega,\epsilon)$ and $\|W^1_t(z,a)\|_{\infty} > C\epsilon$, we have that 
    \begin{equation}\label{eq:W-bound}
        W^1_{t+1}(z,a) = (1 - \alpha_t(z,a)) W^1_t(z,a) + \alpha_t(z,a) F^1_t(z,a) 
         < \| W^1_t \|_{\infty}.
    \end{equation}
    Hence, when $\|W^1_t\|_{\infty} > C\epsilon$, it decreases monotonically. So, there are two possibilities: either it gets below $C\epsilon$ or it never goes below $C\epsilon$. In the supplementary material, we show that the process cannot stay above $C\epsilon$ all the time. Hence, it must hit below $C\epsilon$ at some point.
    In the supplementary material, we show that once the process hits below $C\epsilon$, it stays there. Thus, we have shown that for all sufficiently large $t$, $\|W^1_t\|_{\infty} < C \epsilon$, a.s. Since $\epsilon$ is arbitrary, this implies that $W^1_t(z,a) \to 0$, a.s., for all $(z,a)$. Combining this with the fact that $W^3_t(z,a) \to 0$, we get that $\Delta_t(z,a) = W^1_t(z,a) + W^2_t(z,a) \to 0$, a.s., for all $(z,a)$.
     Let $\mathcal F_t = \sigma(W^1_{1:t}, F^1_{1:t}, \alpha_{1:t})$. Then, 
     \(
     \| \EXP[ F^1_t(z,a) \mid \mathcal F_t] \|_{\infty} \le \EXP[ \| F^1_t(z,a) \|_{\infty} \mid \mathcal F_t ] 
     \le \gamma \| \Delta_t \|_{\infty}.
     \)
     Moreover, since the state space is finite, the per-step reward is bounded. Therefore, both $\hat V_t(Z_{t+1})$ and $V^\star_\xi(Z_{t+1})$ are bounded. Therefore, there exists a constant $C$ such that 
     \(
        \text{var}( F^1_t(z,a) \mid \mathcal F_t ) \le C.
     \)
     Thus, the iteration for $W^1_t(z,a)$ satisfies all conditions of \cite[Theorem 1]{jaakkola1993convergence}, and by that result, $W^1_t(z,a) \to 0$, a.s., for all $(z,a)$. For convenience, we restate \cite[Theorem 1]{jaakkola1993convergence} in the supplementary material. 
\end{proof}

\begin{remark}\label{rem:kara}
    In the special case when $z_t = (y_{t-n:t}, a_{t-n:t-1})$ is the last $n$ observations and actions (i.e., frame-stacking), the result of \autoref{thm:convergence} recovers the result of \cite[Theorem 4.1, (i)]{kara2022convergence}.
\end{remark}
\begin{remark}\label{rem:chandak}
    In \cite[Theorem 1]{chandak2022reinforcement}, it was shown that the results of \autoref{thm:convergence} hold if \textbf{(A2)} and \textbf{(A3)} are replaced by the following:
   \begin{enumerate}[leftmargin=2.5em,topsep=0pt,partopsep=0pt,noitemsep]
       \item [\textbf{(A2')}] Assumption \textbf{(A2)} holds and the initial distribution satisfies $\PR(S_1 = s, Z_1 = z) = \xi(s,z)$.
       \item [\textbf{(A3')}] The learning rate $\alpha_t(z,a)$ satisfies 
       \(
            \sum_{t \ge 1} \alpha_t(z,a) = \infty
       \)
       and
       \(
            \sum_{t \ge 1} \alpha_t^2(z,a) < \infty
       \),
       a.s.
   \end{enumerate}
   Note that \textbf{(A2')} is stronger than \textbf{(A2)} and implies that the process $\{(S_t, Y_t, Z_t, A_t)\}_{t \ge 1}$ is stationary and ergodic.
   However, the learning rate condition \textbf{(A3')} is weaker than \textbf{(A3)}. 
\end{remark}

\autoref{thm:convergence} addresses the main challenge in convergence analysis of RQL: the non-Markovian dynamics. We have shown that RQL converges. So, the next question is: How good is the converged solution compared to the optimal? We address this question in the next section.

\subsection{Characterizing the approximation error of the converged value}\label{sec:approx}

\begin{extra}
\begin{lemma}\label{lem:indep}
    The posterior beliefs $\PR(S_t = \cdot \mid H_t = h_t)$ and $\PR(Z_{t+1} = \cdot \mid H_t = h_t, A_t = a_t)$ do not depend on the exploration policy $\EXPLORE$. 
\end{lemma}
\begin{proof}
    It is a standard result in the POMDP literature that the posterior belief $\PR(S_t = \cdot \mid H_t = h_t)$ does not depend on the policy. See, for example, \red{[\cite{KumarVaraiya:book} ADD Kumar Varaiya book]}. Now consider
    \begin{align*}
        \PR(z_{t+1} \mid h_t, a_t) = \smashoperator[r]{\sum_{s_t \in \ALPHABET S, s_{t+1} \in \ALPHABET S, y_{t+1} \in \ALPHABET Y}} \PR(s_t \mid h_t) P(s_{t+1} \mid s_t, a_t)  O(y_{t+1} \mid s_{t+1}, a_t) \IND_{\{ z_{t+1} = f(\sigma_t(h_t), y_{t+1}, a_t)\}}
    \end{align*}
    Since each term in the right hand side does not depend on the policy, we have that $\PR(Z_{t+1} = \cdot \mid H_t = h_t, A_t = a_t)$ is also policy independent. 
\end{proof}
\end{extra}

Our key observation is that for any $\F$, $( \sigma_t, P_\xi, r_\xi )$ is an $(\epsilon_t, \delta_t)_{t \ge 1}$ AIS-generator where
\begin{align*}
    \varepsilon_t &\coloneqq \max_{h_t \in \ALPHABET H_t, a_t \in \ALPHABET A}
    \bigl| \EXP[ r(S_t, A_t) \mid H_t = h_t, A_t = a_t] 
    - r_\xi(\sigma_t(h_t), a_t) 
    \bigr|, \\
    \delta_t &\coloneqq \max_{h_t \in \ALPHABET H_t, a_t \in \ALPHABET A}
    d_{\F}\bigl( \PR(Z_{t +1} = \cdot \mid H_t = h_t, A_t = a_t), 
            P_{\xi} (\cdot \mid \sigma_t(h_t), a_t) \bigr).
\end{align*}
Therefore, an immediate implication of \autoref{thm:ais} is the following.
\begin{theorem}\label{thm:bound}
    Let $\tilde \pi = (\tilde \pi_1, \tilde \pi_2, \dots)$ be a time-varying and history-dependent policy  given by $\tilde \pi_t(h_t) = \pi^\star_\xi(\sigma_t(h_t))$. Then, for any time~$t$ and any history $h_t \in \ALPHABET H_t$ and action $a_t \in \ALPHABET A$, we have:
    \begin{itemize}[leftmargin=1em, topsep=0pt, partopsep=0pt]
    \item \textbf{Bounds on value approximation:}
    \begin{align}
        \bigl| Q^\star_t(h_t, a_t) - Q^\star_\xi(\sigma_t(h_t), a_t) \bigr| &\le 
        (1-\gamma)^{-1}\bigl[ \bar \varepsilon_t + \gamma \bar \delta_t \rho_{\F}(V^\star_\xi) \bigr],
        \\
        \bigl| V^\star_t(h_t) - V^\star_\xi(\sigma_t(h_t)) \bigr| &\le 
        (1-\gamma)^{-1}\bigl[ \bar \varepsilon_t + \gamma \bar \delta_t \rho_{\F}(V^\star_\xi) \bigr],
    \end{align}
    where
    \(
        \bar \varepsilon_t = (1-\gamma) \sum_{\tau = t}^\infty \gamma^{\tau -t }\varepsilon_{\tau}
    \)
    and
    \(
        \bar \delta_t = (1-\gamma) \sum_{\tau = t}^\infty \gamma^{\tau -t }\delta_{\tau}.
    \)
    \item \textbf{Bounds on policy approximation:}
    \begin{equation} \label{eq:fixed-point-bound} 
        \bigl| V^\star_t(h_t) - V^{\tilde \pi}_t(h_t) \bigr| \le 
        2 (1-\gamma)^{-1}\bigl[  \bar \varepsilon_t +  \gamma \bar \delta_t \rho_{\F}(V^\star_\xi) \bigr].
    \end{equation}
    \end{itemize}
\end{theorem}
\begin{remark}\label{rem:upper}
    We may upper bound 
    \( 
    \bar \varepsilon_t
    \)
    by
    \(\bar \varepsilon^\circ_t \coloneqq \sup_{\tau \ge t} \varepsilon_\tau
    \)
    and
    \(
    \bar \delta_t
    \)
    by 
    \(
     \bar \delta^\circ_t \coloneqq \sup_{\tau \ge t} \delta_\tau
    \).
\end{remark}

The bound of \autoref{thm:bound} is \emph{instance dependent} because it depends on the value function $V^\star_\xi$. In the supplementary material, we illustrate \emph{instance independent} bounds by upper bounding $\rho_{\F}(V^\star_\xi)$ in terms of properties of the transition and reward functions.

\begin{extra}
It is possible to obtain an \emph{instance independent} upper bound on $2 \eta_{\xi,t}$ by upper bounding $\rho_{\F}(V^\star_\xi)$ as we show below.
\begin{proposition}
    For specific choice of $\F$, instance independent upper bounds on the approximation error are given as follows:
    \begin{enumerate}
    \item  When $\F = \F_{\TV}$, $\rho_{\F}(V^\star_\xi) \leq \SPAN(r) / (1-\gamma)$.
    \item  When $\F = \F_{\Was}$, $\rho_{\F}(V^\star_\xi) = \LIP(V^\star_\xi)$. If $r_{\xi}, P_{\xi}$ are Lipschitz functions with Lipschitz constants $\LIP(r_{\xi}), \LIP(P_{\xi})$ and $\gamma \LIP(P_{\xi}) < 1$, then $\LIP(V^\star_\xi) \leq \LIP(r_{\xi}) / (1-\gamma \LIP(P_{\xi}))$
    \end{enumerate}
\end{proposition}
\end{extra}

\section{RQL with AIS losses}

The results of \autoref{thm:bound} suggest that the performance of RQL could be enhanced by improving the representation function~$f$ so as to minimize the approximation losses $\varepsilon = (\varepsilon_1, \varepsilon_2, \dots)$ and $\delta = (\delta_1, \delta_2, \dots)$. A similar idea was proposed in \cite{subramanian2022approximate} to improve the performance of actor-critic algorithms. In this section, we verify that adding such \emph{AIS losses} improves the performance of RQL. 

\subsection{Adding AIS losses to RQL}
The key idea is to model each component of the AIS generator as a parametric family of functions/distribution and then use stochastic gradient descent to update these parameters. Note that in RQL, we use a RNN to model the state update function $f$. As explained in \cite[Proposition 8]{subramanian2022approximate}, in this case we can use an observation predictor $\tilde P^y \colon \ALPHABET Z \times \ALPHABET A \to \Delta(\ALPHABET Y)$ as a proxy for the state predictor $\tilde P$ and replace $\delta_t$ by $\tilde \delta_t/\bar \kappa_{\F}(f)$, where
\[
    \tilde \delta_t \coloneqq \max_{h_t \in \ALPHABET H_t, a_t \in \ALPHABET A}
    d_{\F}\bigl( \PR(Y_{t +1} = \cdot \mid H_t = h_t, A_t = a_t), 
            \tilde P^y(\cdot \mid \sigma_t(h_t), a_t) \bigr).
\]
and $\bar \kappa_{\F}(f) = \sup_{z \in \ALPHABET Z,a \in \ALPHABET A} \kappa_{\F}( f( z, \cdot, a))$. 

Let $\psi$ denote the combined parameters of the AIS-generator. Then, we could choose 
the AIS loss function $\mathcal{L}_{\text{AIS}}$ as any monotonic function of $(\varepsilon_t, \delta_t)$ because reducing $(\varepsilon_t,\delta_t)$ reduces the upper bound of \autoref{thm:bound}. As suggested in \cite{subramanian2022approximate}, we choose the AIS loss as
\[
    \mathcal{L}_{\text{AIS}}(\psi) = \frac{1}{T} \sum_{t=1}^T (\lambda \varepsilon^2 + (1-\lambda) \tilde \delta^2)
\]
where $\lambda$ is a hyper-parameter and $T$ is the batch size. A detailed discussion of the choice of IPM is presented in \cite{subramanian2022approximate}. 
We choose $d_{\F}$ as $\ell_2$-distance based MMD \cite[see][Proposition 32]{subramanian2022approximate} for which case  the AIS loss can be simplified as~\cite[Proposition 33]{subramanian2022approximate}:
\[
    \mathcal{L}_{\text{AIS}}(\psi) = \frac{1}{T} \sum_{t=1}^T \Bigl[ \lambda \bigl| R_t - \tilde r(Z_t, A_t) \bigr|^2  + (1-\lambda)( M^y_t - 2Y_t)^\intercal M^y_t  \Bigr]+ \text{const}(\psi)
\]
where $M^y_t$ is the mean of the distribution $\tilde P^y(\cdot \mid Z_t, A_t)$ and $\text{const}(\psi)$ are terms which do not depend on $\psi$ and can therefore be ignored.  Then, we update the parameters $\psi$ of the AIS generator using stochastic gradient descent.

\begin{figure}[tb] 
    \centering
    \includegraphics[width=\textwidth]{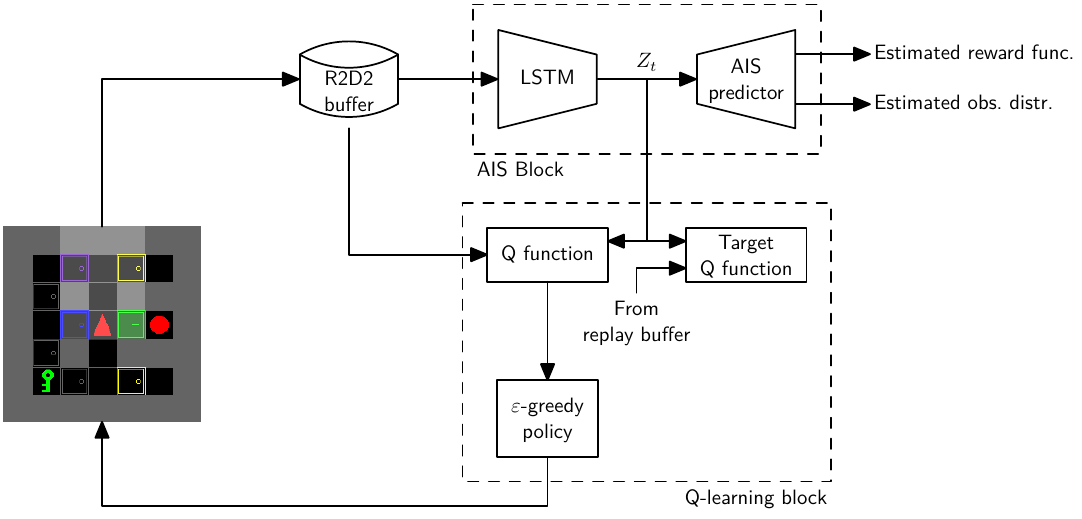}
    \caption{Network architecture of RQL with AIS losses}
    \label{fig:diagram}
\end{figure}

The updates of the representation using AIS losses is carried out in parallel with RQL. For our experiments, we choose R2D2~\cite{kapturowski2018recurrent}, which generalizes Double Q-learning (DQL) with replay buffers to RNNs~\cite{https://doi.org/10.48550/arxiv.1509.06461}. 
Note that as in \cite{subramanian2022approximate}, we do not backpropagate the Q-learning losses to the AIS generator. A block diagram showing the network architecture is shown in \autoref{fig:diagram}. The complete implementation details are presented in the supplementary material. 

\subsection{Empirical evaluation}

\begin{wraptable}{r}{0.45\textwidth}
\vskip -3\baselineskip
\caption{Comparison of RQL-AIS and ND-R2D2 on the MiniGrid benchmark}
\label{tab:comparison}
\includegraphics{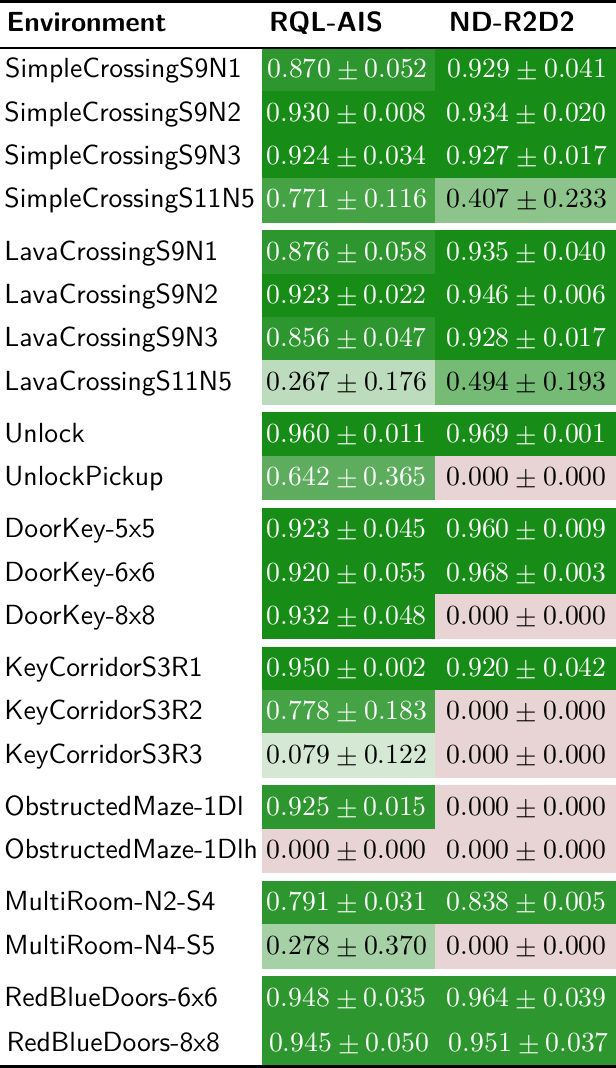}
\vskip -2\baselineskip
\end{wraptable}

In this section, 
we compare the performance of \texttt{RQL-AIS} (the algorithm described in the previous section) with a non-distributed\footnote{In~\cite{kapturowski2018recurrent}, the distributed setup is introduced to allow for parallel data collection and training in environments requiring very high number of interactions. This is not the case in our setup and both the RQL-AIS and R2D2 are implemented in a non-distributed manner.} variant of R2D2 proposed in~\cite{kapturowski2018recurrent}, which we label as \texttt{ND-R2D2}. The exact implementation details, including the choice of hyper-parameters are presented in the supplementary material. We want to emphasize that the two implementations are identical except that \texttt{RQL-AIS} includes the AIS loss block and updates the parameters $\psi$ of the representation by backpropogating the AIS loss $\mathcal{L}_{\text{AIS}}(\psi)$ rather than backpropogating the Q-learning losses.

We evaluate the two algorithms on 22 environments from the MiniGrid benchmark \cite{minigrid}, which are partially observed MiniGrid environments with tasks with increasing complexity, described in detail in the supplementary material.  In all environments, the layout at the start of each episode is randomly chosen, so the agent cannot solve the task by memorizing an exact sequence of actions.  Similar to \cite{subramanian2022approximate,ha2018worldmodels}, before running the experiments, we train an autoencoder on a dataset of random agent observations to compress the observations of the agent into compact vectors. The weights of the autoencoder are kept frozen during the RL experiments.

\begin{figure*}[b] 
\centering
\includegraphics[width=\textwidth]{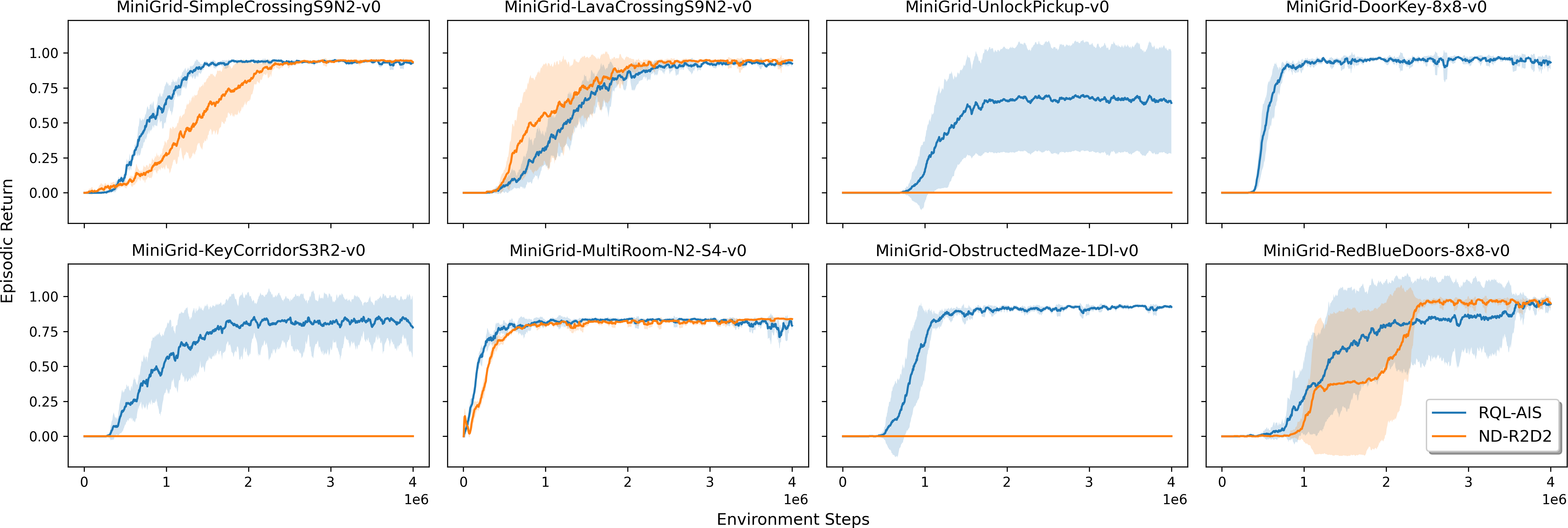}
\caption{Results from 8 selected MiniGrid environments. RQL-AIS successfully solves all 8 while ND-R2D2 fails at solving 4 environments. This demonstrates the effectiveness of AIS state representations in empirical settings.}
\label{fig:performance}
\end{figure*}

We train both \texttt{RQL-AIS} and \texttt{ND-R2D2} for $T = 4 \cdot 10^{6}$ environment steps for $N=5$ seeds. 
During the data gathering phase, the agent behaves according to the epsilon-greedy approach \cite{mnih2013atari} to allow for exploration of the environment with epsilon value exponentially decreasing over time. We run two set of studies: one with uniform sampling from the replay buffer and the other using prioritized  experience replay (PER) \cite{https://doi.org/10.48550/arxiv.1511.05952}. We report the results for uniform sampling here and the results for PER in the supplementary material.
 
The mean and standard deviation of the final performance for both algorithms is shown in \autoref{tab:comparison}. \autoref{fig:performance} shows the training curves for 8 representative environments. Training curves for all environments are included in the supplementary material. 

\paragraph{Discussion of the results} For the simpler environments, both \texttt{RQL-AIS} and \texttt{ND-R2D2} learn to solve the task, with \texttt{ND-R2D2} performing slightly better. However, there are seven environments where \texttt{ND-R2D2} fails to learn. This is not surprising as the MiniGrid environments are sparse reward environments 
which are used as a benchmark for research on exploration methods in RL \cite{parisi2021interesting,pmlr-v162-mavor-parker22a,zhang2021bebold,jiang2021prioritized,alshedivat2018complexity,goyal2018transfer}. \texttt{ND-R2D2}, which does not include any exploration bonuses, fails to learn in some of the larger environments. What is surprising is that \texttt{RQL-AIS} is still able to learn in such environments without any exploration bonuses. We show in the supplementary material that adding prioritized experience replay boosts the performance of \texttt{RQL-AIS} in the harder environments.

\paragraph{What is the impact of AIS losses on learning?} 
To understand the impact of AIS losses on learning, we compare the evolution of the MMD distance and the episodic return with the number of environment steps. The results are shown in \autoref{fig:alblation}, where the environment steps are plotted on a log scale. In each environment, when the MMD loss is initially high, the episodic return does not improve considerably. However, when the MMD distance loss becomes smaller ($\approx 0.5$), the episodic return starts improving. This suggests that it takes some time for the agent to learn a good representation through the AIS losses, and once a good representation has been obtained (small MMD loss), RL losses through Q-learning are much more effective in training policies, thus allowing policies to improve much quicker, i.e., with fewer samples. This makes sense because optimizing the AIS losses effectively gives the same representation to several different possible histories (but identical from the perspective of performance), which then allows Q-learning updates to propagate through all these possible histories instead of just a single history trajectory. The ability to map more histories to fewer representations accurately improves with the quality of approximation of the AIS. Thus, AIS losses not only help in establishing theoretical performance bounds, but it is evident that they help in learning in empirical experiments.

\begin{figure*}[tb] 
\centering
\includegraphics[width=\textwidth]{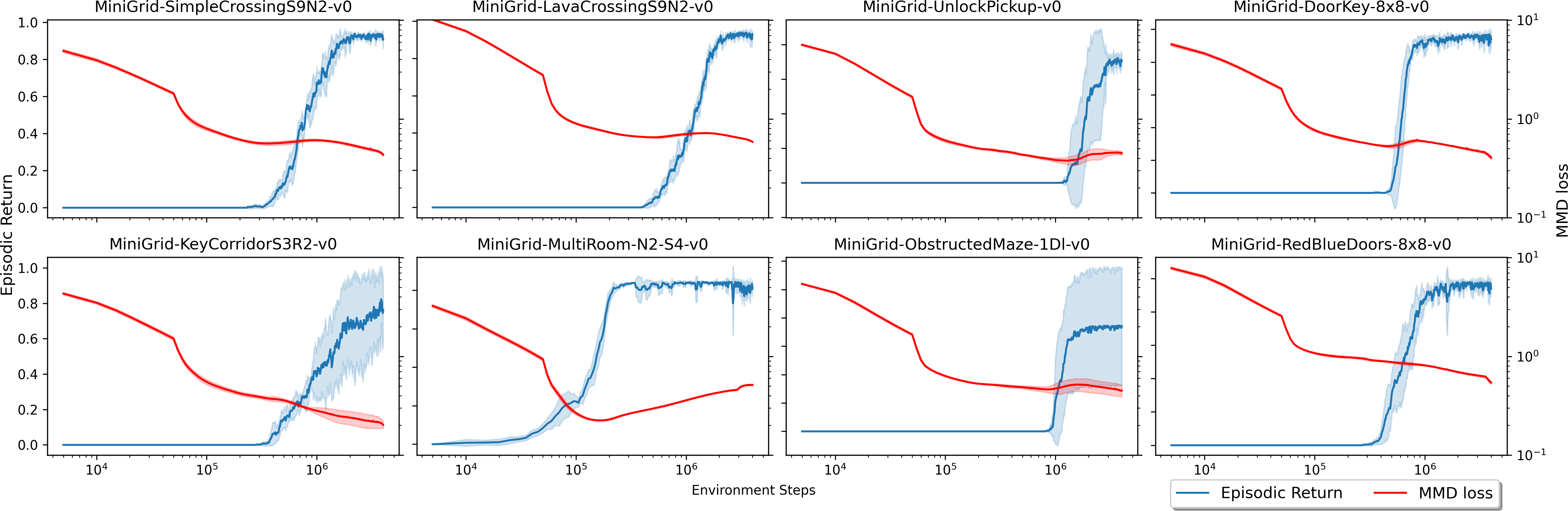}
\caption{Episodic Return (left-axis) and MMD loss (right-axis, on log scale) plots for the RQL-AIS method on 8 selected MiniGrid environments. There is a close correlation between the drop in MMD loss value and improvement in episodic return. Both Environment Steps and MMD loss are in the logarithmic scale. }
\label{fig:alblation}
\end{figure*}

\section{Conclusion}
In this work, we establish the convergence of recurrent Q-learning (RQL) in the tabular setting for POMDPs using a representation of the history called an approximate information state (AIS). We also establish upper bounds on the degree of sub-optimality of the converged solution. These bounds quantify the relationship between the quality of representation and the quality of the converged solution of RQL. Based on these bounds, a variant of RQL called RQL-AIS was proposed and it was observed that RQL-AIS performs better than the state-of-the-art baseline ND-R2D2 on the MiniGrid benchmark. A detailed comparison of the time evolution of AIS losses and performance strongly suggests that the improvement in performance is correlated to the decrease in AIS losses. In conclusion, the results of this paper show that AIS is a useful theoretical tool for analyzing the performance of RQL in POMDPs and also an effective block for algorithm design for POMDPs.

Our theoretical analysis is restricted to the tabular setting. Interesting future work includes generalizing the analysis to more general models, including using function approximation for approximating the Q-function. Another interesting avenue is formally showing convergence of an algorithm which learns the AIS and Q-function in parallel with two time-scale methods.


\section*{Acknowledgements}

This work was supported in part by Natural Science and Engineering Council of Canada through Discovery Grant RGPIN-2021-03511 and Alliance International Catalyst Grant ALLRP 571054-21. The numerical experiments were enabled in part by  compute resources provided by the Digital Research Alliance of Canada.

\printbibliography

\clearpage

\appendix
\section{Proof of claims in Theorem 1}

The high-level idea of the proof is the same as that \cite{subramanian2022approximate} [Theorem 27] and relies on the following two observations:
\begin{enumerate}[leftmargin=1em]
    \item  \textbf{Approximating the infinite horizon system by a finite horizon system.}
    First note that the finiteness of the state and action spaces implies that the per-step reward is uniformly bounded. Define $r_{\text{MIN}} = \min_{s \in \ALPHABET S, a \in \ALPHABET A}r(s,a)$ and $r_{\text{MAX}} = \max_{s \in \ALPHABET S, a \in \ALPHABET A} r(s,a)$. Then,  $R_t \in [r_{\textup{MIN}}, r_{\textup{MAX}}]$.  

    Now consider a finite horizon system that runs for horizon $T$. Let $V^\star_{t,T}$ and $V^{\tilde \pi}_{t,T}$ denote the optimal finite-horizon value function and the finite-horizon value function corresponding to policy $\tilde \pi$ (defined in \autoref{thm:ais}), respectively. Moreover, let $Q^\star_{t,T}$ denote the optimal finite-horizon Q-function. More specifically, for any $t \le T$, and any history $h_t \in \ALPHABET H_t$ and action $a_t \in \ALPHABET A$, we have 
    \begin{align*}
        V^{\star}_{t,T}(h_t) &\coloneqq \sup_{\pi} \EXP^{\pi} \left[ \medop\sum_{\tau=t}^{T-1} \gamma^{\tau-t} r(S_{\tau}, A_{\tau}) \Bigm| H_t = h_t \right],\\
        V^{\tilde \pi}_{t,T}(h_t) &\coloneqq \EXP^{\tilde \pi} \left[ \medop\sum_{\tau=t}^{T-1} \gamma^{\tau-t} r(S_{\tau}, A_{\tau}) \Bigm| H_t = h_t \right],\\
        Q^{\star}_{t,T}(h_t, a_t) &\coloneqq \sup_{\pi} \EXP^{\pi} \left[  \medop\sum_{\tau=t}^{T-1} \gamma^{\tau-t} r(S_{\tau}, A_{\tau}) \bigm| H_t = h_t, A_t = a_t \right].
    \end{align*}

    Then, the boundedness of the per-step reward implies the following:
    \begin{proposition} \label{prop:finite-inf-sandwich-bounds}
        For any $T$, any $t \le T$, and any history $h_t \in \ALPHABET H_t$ and action $a_t \in \ALPHABET A$, we have the following:
        \begin{alignat}{4}
           Q^\star_{t,T}(h_t, a_t) + \frac{\gamma^{T-t}}{1-\gamma} r_{\textup{MIN}}
           & {} \le {} & Q^\star_t(h_t, a_t) & {} \le {}
           Q^\star_{t,T}(h_t, a_t) + \frac{\gamma^{T-t}}{1-\gamma} r_{\textup{MAX}},
           \label{eq:finite_inf_horizon_sandwich_bound_Q} \\
           V^\star_{t,T}(h_t) + \frac{\gamma^{T-t}}{1-\gamma} r_{\textup{MIN}}
           & {} \le {} & V^\star_t(h_t) & {} \le {}
           V^\star_{t,T}(h_t) + \frac{\gamma^{T-t}}{1-\gamma} r_{\textup{MAX}},
           \label{eq:finite_inf_horizon_sandwich_bound_V} \\
           V^{\tilde \pi}_{t,T}(h_t) + \frac{\gamma^{T-t}}{1-\gamma} r_{\textup{MIN}}
           & {} \le {} & V^{\tilde \pi}_t(h_t) & {} \le {}
           V^{\tilde \pi}_{t,T}(h_t) + \frac{\gamma^{T-t}}{1-\gamma} r_{\textup{MAX}}. \label{eq:finite_inf_horizon_sandwich_bound_pol_bound}
        \end{alignat}
    \end{proposition}

    \item \textbf{Approximating the finite-horizon system with the approximate AIS model.}
        Given the AIS reward $\tilde r$ and dynamics $\tilde P$, define a Bellman operator $\tilde{\mathcal{B}}_V \colon \ALPHABET Z \to \ALPHABET Z$ as follows. For any $\tilde V \colon \ALPHABET Z \to \reals$, 
        \[
            \tilde{\mathcal{B}}_V \tilde V(z) = \max_{a \in \ALPHABET A} \left\{ \tilde r(z, a) + \gamma \int_{\ALPHABET Z} \tilde V(z') \tilde P(dz' \mid z, a)\right\},
        \]
        and define a Bellman operator $\tilde{\mathcal{B}}_Q \colon \ALPHABET Z \times \ALPHABET A \to \ALPHABET Z \times \ALPHABET A$ as follows. For any $\tilde Q \colon \ALPHABET Z \times \ALPHABET A \to \reals$,
        \begin{align*}
            \tilde V(z) &= \max_{a \in \ALPHABET A} \tilde Q(z, a) \\
            \tilde{\mathcal{B}}_Q \tilde Q(z, a) &= \tilde r(z, a) + \gamma \int_{\ALPHABET Z} \max_{a' \in \ALPHABET A} \tilde Q(z', a') \tilde P(dz' \mid z, a)\\
             &= \tilde r(z, a) + \gamma \int_{\ALPHABET Z} \tilde V(z') \tilde P(dz' \mid z, a).
        \end{align*}
        Thus the two Bellman operators $\tilde{\mathcal{B}}_V$ and $\tilde{\mathcal{B}}_Q$ are related as follows
        \begin{align*}
            \tilde{\mathcal{B}}_V \tilde V(z) = \max_{a \in \ALPHABET A} \tilde{\mathcal{B}}_Q \tilde Q(z, a).
        \end{align*}
        
    Now, define initial value functions $\tilde V^0(z) = 0$ and $\tilde Q^0(z, a) = 0$ for all $z$ and $a$. Recursively define $\tilde V^{n+1} = \tilde {\mathcal{B}}_V \tilde V^n$ and $\tilde Q^{n+1} = \tilde {\mathcal{B}}_Q \tilde Q^n$.  Then, we have the following.
    \begin{proposition} \label{prop:finite-horizon-bounds}
    Arbitrarily fix the function space $\F$.  For any $T$ and $t \le T$, define
    \[
    \eta_{t,T} = \sum_{\tau=t}^{T-1} \gamma^{\tau-t} \varepsilon_{\tau} + \sum_{\tau=t+1}^{T-1} \gamma^{\tau-t} \rho_{\mathfrak{F}}(\tilde V^{T-\tau}) \delta_{\tau-1}.
    \]
    Then, we have the following bounds:
    \begin{align}
        \lvert Q^{\star}_{t,T}(h_{t}, a_{t}) - \tilde Q^{T-t}(\sigma_{t}(h_{t}), a_{t}) &\rvert \leq \eta_{t,T}, \label{eq:finite_horizon_AIS_bound_Q} \\
        \lvert V^{\star}_{t,T}(h_{t}) - \tilde V^{T-t}(\sigma_{t}(h_{t})) &\rvert \leq \eta_{t,T}, \label{eq:finite_horizon_AIS_bound_V} \\
        \lvert V^{\tilde \pi}_{t,T}(h_{t}) - \tilde V^{T-t}(\sigma_{t}(h_{t})) &\rvert \leq \eta_{t,T}. \label{eq:finite_horizon_AIS_bound_optpol}
    \end{align}
    \end{proposition}
    
\end{enumerate}

\subsection{Proof of Proposition~\ref{prop:finite-inf-sandwich-bounds}}
Consider the optimal value functions in the original model for the infinite horizon case, i.e, $V^{\star}_t(h_t)$ and $Q^{\star}_t(h_t, a_t)$.
\begin{align*}
    V^{\star}_t(h_t) &= \sup_{\pi} \EXP \left[ \sum_{\tau=t}^{\infty} \gamma^{\tau-t} R_{\tau} \,\middle|\, H_t = h_t \right] \\
    &\geq \sup_{\pi} \EXP \left[ \sum_{\tau=t}^{T-1} \gamma^{\tau-t} R_{\tau} + \sum_{\tau=T}^{\infty} \gamma^{\tau-t} r_{\textup{MIN}} \,\middle|\, H_t = h_t \right] \\
    &= \sup_{\pi} \EXP \left[ \sum_{\tau=t}^{T-1} R_{\tau} \,\middle|\, H_t = h_t \right] + \frac{\gamma^{T-t}}{1-\gamma} r_{\textup{MIN}} \\
    &= V^{\star}_{t,T}(h_t) + \frac{\gamma^{T-t}}{1-\gamma} r_{\textup{MIN}}.
\end{align*}
Similarly, it can be shown by reversing the inequality and using $r_{\textup{MAX}}$ instead of $r_{\textup{MIN}}$ that
\begin{align*}
    V^{\star}_t(h_t) \leq V^{\star}_{t,T}(h_t) + \frac{\gamma^{T-t}}{1-\gamma} r_{\textup{MAX}}.
\end{align*}

Thus, we have the result
\begin{align*} 
   V^{\star}_{t,T}(h_t) + \frac{\gamma^{T-t}}{1-\gamma} r_{\textup{MIN}} \leq V^{\star}_t(h_t) \leq V^{\star}_{t,T}(h_t) + \frac{\gamma^{T-t}}{1-\gamma} r_{\textup{MAX}}.
\end{align*}

We can obtain a similar bound for $Q^\star_t$. In particular,  we have
\begin{align*}
    Q^{\star}_t(h_t, a_t) &= \sup_{\pi} \EXP \left[ r(h_t, a_t) + \sum_{\tau=t+1}^{\infty} \gamma^{\tau-t} R_{\tau} \,\middle|\, H_t = h_t, A_t = a_t \right] \\
    &\geq \sup_{\pi} \EXP \left[ r(h_t, a_t) + \sum_{\tau=t+1}^{T-1} \gamma^{\tau-t} R_{\tau} + \sum_{\tau=T}^{\infty} \gamma^{\tau-t} r_{\textup{MIN}} \,\middle|\, H_t = h_t, A_t = a_t \right] \\
    &= \sup_{\pi} \EXP \left[ r(h_t, a_t) + \sum_{\tau=t+1}^{T-1} R_{\tau} \,\middle|\, H_t = h_t, A_t = a_t \right] + \frac{\gamma^{T-t}}{1-\gamma} r_{\textup{MIN}} \\
    &= Q^{\star}_{t,T}(h_t, a_t) + \frac{\gamma^{T-t}}{1-\gamma} r_{\textup{MIN}}.
\end{align*}
Similarly, it can be shown by reversing the inequality and using $r_{\textup{MAX}}$ instead of $r_{\textup{MIN}}$ that
\begin{align*}
    Q^{\star}_t(h_t, a_t) \leq Q^{\star}_{t,T}(h_t) + \frac{\gamma^{T-t}}{1-\gamma} r_{\textup{MAX}}.
\end{align*}

Thus, we have
\begin{align*} 
   Q^{\star}_{t,T}(h_t, a_t) + \frac{\gamma^{T-t}}{1-\gamma} r_{\textup{MIN}} \leq Q^{\star}_t(h_t, a_t) \leq Q^{\star}_{t,T}(h_t, a_t) + \frac{\gamma^{T-t}}{1-\gamma} r_{\textup{MAX}}.
\end{align*}

Next, consider the value function in the original model for the infinite horizon case $V^{\tilde \pi}_t(h_t)$.
\begin{align*}
    V^{\tilde \pi}_t(h_t) &= \EXP^{\tilde \pi} \left[ \sum_{\tau=t}^{\infty} \gamma^{\tau-t} R_{\tau} \mid H_t = h_t \right] \\
    &\geq \EXP^{\tilde \pi} \left[ \sum_{\tau=t}^{T-1} \gamma^{\tau-t} R_{\tau} + \sum_{\tau=T}^{\infty} \gamma^{\tau-t} r_{\textup{MIN}} \mid H_t = h_t \right] \\
    &= \EXP^{\tilde \pi} \left[ \sum_{\tau=t}^{T-1} R_{\tau} \mid H_t = h_t \right] + \frac{\gamma^{T-t}}{1-\gamma} r_{\textup{MIN}} \\
    &= V^{\tilde \pi}_{t,T}(h_t) + \frac{\gamma^{T-t}}{1-\gamma} r_{\textup{MIN}}.
\end{align*}

Similarly, it can be shown by reversing the inequality and using $r_{\textup{MAX}}$ instead of $r_{\textup{MIN}}$ that
\begin{align*}
    V^{\tilde \pi}_t(h_t) \leq V^{\tilde \pi}_{t,T}(h_t) + \frac{\gamma^{T-t}}{1-\gamma} r_{\textup{MAX}}.
\end{align*}

Thus, we have the result
\begin{align*} 
   V^{\tilde \pi}_{t,T}(h_t) + \frac{\gamma^{T-t}}{1-\gamma} r_{\textup{MIN}} \leq V^{\tilde \pi}_t(h_t) \leq V^{\tilde \pi}_{t,T}(h_t) + \frac{\gamma^{T-t}}{1-\gamma} r_{\textup{MAX}}.
\end{align*}

\subsection{Proof of Proposition~\ref{prop:finite-horizon-bounds}}

First, we look at the proof for \eqref{eq:finite_horizon_AIS_bound_Q} and \eqref{eq:finite_horizon_AIS_bound_V}. A dynamic program can be written for the optimal value functions for $t\in\{1, \cdots, T-1\}$ as $\left\{ V^{\star}_{t,T} \colon \ALPHABET H_t \to \reals \right\}$ and $\left\{ Q^{\star}_{t,T} \colon \ALPHABET H_t \times \ALPHABET A_t \to \reals \right\}$ with $V^{\star}_{T,T} (h_T) = 0$ and $Q^{\star}_{T,T} (h_T, a_T) = 0$ such that:

\begin{align*}
    V^{\star}_{t,T}(h_t) &\coloneqq \max_{a_t \in \ALPHABET A} \EXP \left[ R_t + \gamma V^{\star}_{t+1,T}(H_{t+1}) \mid H_t = h_t, A_t = a_t \right], \\
    Q^{\star}_{t,T}(h_t, a_t) &\coloneqq \EXP \left[ R_t + \gamma V^{\star}_{t+1,T}(H_{t+1}) \mid H_t = h_t, A_t = a_t \right].
\end{align*}

We prove the result by backward induction. Consider the bounds $\lvert Q^{\star}_{T,T}(h_T, a_T) - \tilde Q^{0}(\sigma_T(h_T), a_T) \rvert \leq \eta_{T,T} = 0$ and $\lvert V^{\star}_{T,T}(h_T, a_T) - \tilde V^{0}(\sigma_T(h_T), a_T) \rvert \leq \eta_{T,T} = 0$ which hold by definition. Thus, the bound holds for $t=T$.  This forms the basis of induction. Now, assume that the bound holds for $t+1$ and consider the system at time $t$. We have
\begin{align*}
    \lvert Q^{\star}_{t,T}(h_t, a_t) &- \tilde Q^{T-t}(\sigma_t(h_t), a_t) \rvert \\
    &\stackrel{(a)}\leq \lvert \EXP[R_t \mid H_t = h_t, A_t = a] - \tilde r(\sigma_t(h_t), a) \rvert \\
    & \quad + \gamma \EXP [\lvert V^{\star}_{t+1,T}(H_{t+1}) - \tilde V^{T-t-1}(\sigma_{t+1}(H_{t+1})) \rvert \mid H_t = h_t, A_t = a] \\
    & \quad + \gamma \left\lvert \EXP[\tilde V^{T-t-1}(\sigma_{t+1}(H_{t+1})) \mid H_t = h_t, A_t = a] - \int_{\ALPHABET Z} \tilde V^{T-t-1}(z_{t+1}) \tilde P(dz_{t+1} \mid \sigma_t(h_t), a) \right\rvert \\
    &\stackrel{(b)}\leq \varepsilon_t + \gamma \eta_{t+1,T} + \gamma \rho_{\mathfrak{F}}(\tilde V^{T-t-1}) \delta_t = \eta_{t,T}
\end{align*}
where $(a)$ follows from the triangle inequality, the definition of $Q^{\star}_{t,T}(h_t, a_t)$, the definition of $\tilde Q^{T-t}(\sigma_t(h_t), a_t)$; $(b)$ follows from the induction hypothesis and  the definitions of $\varepsilon_t$ and $\delta_t$. We also have,

\begin{align*}
    \lvert V^{\star}_{t,T}(h_t) - \tilde V^{T-t}(\sigma_t(h_t)) \rvert &\stackrel{(c)}\leq \max_{a_t \in \ALPHABET A} \lvert Q^{\star}_{t,T}(h_t, a_t) - \tilde Q^{T-t}(\sigma_t(h_t), a_t) \rvert \\
    &\leq \eta_{t,T} = \sum_{\tau=t}^{T-1} \gamma^{\tau-t} \varepsilon_{\tau} + \sum_{\tau=t+1}^{T-1} \gamma^{\tau-t} \rho_{\mathfrak{F}}(\tilde V^{T-\tau}) \delta_{\tau-1}
\end{align*}
where $(c)$ follows from the inequalities $\max f(x) \leq \max \lvert f(x) - g(x) \rvert + \max g(x)$. Thus, the result holds for $t$. Then, by the principle of induction, it holds for all $t \le T$. This completes the proof for \eqref{eq:finite_horizon_AIS_bound_Q} and \eqref{eq:finite_horizon_AIS_bound_V}.

Next, we look at the proof for \eqref{eq:finite_horizon_AIS_bound_optpol}. Consider value functions for $t\in\{1, \cdots, T-1\}$ as $\left\{ V^{\tilde \pi}_{t,T} \colon \ALPHABET H_t \to \reals \right\}$ with $V^{\tilde \pi}_{T,T} (h_T) = 0$ that are recursively defined such that:

\begin{align*}
    V^{\tilde \pi}_{t,T}(h_t) \coloneqq \EXP^{\tilde \pi} \left[ R_t + \gamma V^{\tilde \pi}_{t+1,T}(H_{t+1}) \mid H_t = h_t \right].
\end{align*}

We prove this result by backward induction as well. Consider the bound $\lvert V^{\tilde \pi}_{T,T}(h_T) - \tilde V^{0}(\sigma_T(h_T)) \rvert \leq \eta_{T,T} = 0$ which holds by definition. Thus, the bound holds for $t=T$. This forms the basis of induction.
Now assume that the bound holds for $t+1$ and consider the system at time $t$. We have
\begin{align*}
    \lvert V^{\tilde \pi}_{t,T}(h_t) &- \tilde V^{T-t}(\sigma_t(h_t)) \rvert \\
    &\stackrel{(a)}\leq \max_{a \in \ALPHABET A} \lvert \EXP[R_t \mid H_t = h_t, A_t = a] - \tilde r(\sigma_t(h_t), a) \rvert \\
    & \quad + \gamma \max_{a \in \ALPHABET A} \EXP [\lvert V^{\tilde \pi}_{t+1,T}(H_{t+1}) - \tilde V^{T-t-1}(\sigma_{t+1}(H_{t+1})) \rvert \mid H_t = h_t, A_t = a] \\
    & \quad + \gamma \max_{a \in \ALPHABET A} \left\lvert \EXP[\tilde V^{T-t-1}(\sigma_{t+1}(H_{t+1})) \mid H_t = h_t, A_t = a] - \int_{\ALPHABET Z} \tilde V^{T-t-1}(z_{t+1}) \tilde P(dz_{t+1} \mid \sigma_t(h_t), a) \right\rvert \\
    &\stackrel{(b)}\leq \varepsilon_t + \gamma \eta_{t+1,T} + \gamma \rho_{\mathfrak{F}}(\tilde V^{T-t-1}) \delta_t = \eta_{t,T}
\end{align*}
where $(a)$ follows from the triangle inequality, the definition of $V^{\tilde \pi}_{t,T}(h_t)$, the definition of $\tilde V^{T-t}(\sigma_t(h_t))$, the fact that $\sum_{a \in \ALPHABET A} \tilde \pi (a \mid z) = 1$, and upper-bounding over all actions; $(b)$ follows from the induction hypothesis and the definitions of $\varepsilon_t$ and $\delta_t$.

Thus, this gives us
\begin{align*}
    \lvert V^{\tilde \pi}_{t,T}(h_t) &- \tilde V^{T-t}(\sigma_t(h_t)) \rvert \leq \eta_{t,T} = \sum_{\tau=t}^{T-1} \gamma^{\tau-t} \varepsilon_{\tau} + \sum_{\tau=t+1}^{T-1} \gamma^{\tau-t} \rho_{\mathfrak{F}}(\tilde V^{T-\tau}) \delta_{\tau-1}.
\end{align*}
Thus, the result holds for $t$. Then, by the principle of induction, it holds for all $t \le T$. This completes the proof for  \eqref{eq:finite_horizon_AIS_bound_optpol}.

\subsection{Bounds on value approximation}

Combining Proposition~\ref{prop:finite-inf-sandwich-bounds}: \eqref{eq:finite_inf_horizon_sandwich_bound_Q}, \eqref{eq:finite_inf_horizon_sandwich_bound_V} with Proposition~\ref{prop:finite-horizon-bounds}: \eqref{eq:finite_horizon_AIS_bound_Q}, \eqref{eq:finite_horizon_AIS_bound_V} gives us
\begin{align*}
    \tilde V^{T-t}(\sigma_t(h_t)) - \eta_{t,T} + \frac{\gamma^{T-t}}{1-\gamma} r_{\textup{MIN}} \leq &V^{\star}_t(h_t) \leq \tilde V^{T-t}(\sigma_t(h_t)) + \eta_{t,T} + \frac{\gamma^{T-t}}{1-\gamma} r_{\textup{MAX}}, \\
    \tilde Q^{T-t}(\sigma_t(h_t), a_t) - \eta_{t,T} + \frac{\gamma^{T-t}}{1-\gamma} r_{\textup{MIN}} \leq &Q^{\star}_t(h_t, a_t) \leq \tilde Q^{T-t}(\sigma_t(h_t), a_t) + \eta_{t,T} + \frac{\gamma^{T-t}}{1-\gamma} r_{\textup{MAX}}.
\end{align*}
Note that due to discounting both $\tilde{\mathcal{B}}_V$ and $\tilde{\mathcal{B}}_Q$ are contractions. Therefore, from the Banach fixed point theorem, we know that $\lim_{T \to \infty} \tilde V^{T-t} = \tilde V^{\star}$ and $\lim_{T \to \infty} \tilde Q^{T-t} = \tilde Q^\star$. Furhermore, by the continuity of $\rho_{\mathfrak{F}}(\cdot)$, we have $\lim_{T \to \infty} \rho_{\mathfrak{F}}(\tilde V^{T-t}) = \rho_{\mathfrak{F}}(\tilde V^{\star})$. Thus, taking the limit $T \to \infty$ in the above equations gives us
\begin{align*}
    \tilde V^{\star}(\sigma_t(h_t)) - \eta^{\star}_t \leq V^{\star}_t(h_t) \leq \tilde V^{\star}(\sigma_t(h_t)) + \eta^{\star}_t, \\
    \tilde Q^{\star}(\sigma_t(h_t), a_t) - \eta^{\star}_t \leq Q^{\star}_t(h_t, a_t) \leq \tilde Q^{\star}(\sigma_t(h_t), a_t) + \eta^{\star}_t,
\end{align*}
where, $\eta^{\star}_t = \sum_{\tau=t}^{\infty} \gamma^{\tau-t} \varepsilon_{\tau} + \gamma \rho_{\F}(\tilde V^{\star}) \sum_{\tau=t}^{\infty} \gamma^{\tau-t} \delta_{\tau}$.

This shows that
\begin{align} 
    \lvert V^{\star}_t(h_t) - \tilde V^{\star}(\sigma_t(h_t)) \rvert \leq \eta^{\star}_t, \label{eq:value_approx_main_bound_V} \\ 
    \lvert Q^{\star}_t(h_t, a_t) - \tilde Q^{\star}(\sigma_t(h_t), a_t) \rvert \leq \eta^{\star}_t. \label{eq:value_approx_main_bound_Q}
\end{align}
This establishes the bounds in the main theorem on value approximation.

\subsection{Bounds on policy approximation}

Combining Proposition~\ref{prop:finite-inf-sandwich-bounds}: \eqref{eq:finite_inf_horizon_sandwich_bound_pol_bound} with Proposition~\ref{prop:finite-horizon-bounds}: \eqref{eq:finite_horizon_AIS_bound_optpol} gives us
\begin{align*}
    \tilde V^{T-t}(\sigma_t(h_t)) - \eta_{t,T} + \frac{\gamma^{T-t}}{1-\gamma} r_{\textup{MIN}} \leq V^{\tilde \pi}_t(h_t) \leq \tilde V^{T-t}(\sigma_t(h_t)) + \eta_{t,T} + \frac{\gamma^{T-t}}{1-\gamma} r_{\textup{MAX}},
\end{align*}
Note that due to discounting, $\tilde{\mathcal{B}}_V$ is a contraction. Therefore, from the Banach fixed point theorem, we know that $\lim_{T \to \infty} \tilde V^{T-t} = \tilde V^{\tilde \pi^{\star}}$. Furthermore, by continuity of $\rho_{\mathfrak{F}}(\cdot)$, we have $\lim_{T \to \infty} \rho_{\mathfrak{F}}(\tilde V^{T-t}) = \rho_{\mathfrak{F}}(\tilde V^{\tilde \pi^{\star}})$. Thus, taking the limit $T \to \infty$ in the above equation gives us
\begin{align*}
    \tilde V^{\tilde \pi^{\star}}(\sigma_t(h_t)) - \eta^{\tilde \pi^{\star}}_t \leq V^{\tilde \pi}_t(h_t) \leq \tilde V^{\tilde \pi^{\star}}(\sigma_t(h_t)) + \eta^{\tilde \pi^{\star}}_t,
\end{align*}
where, $\eta^{\tilde \pi^{\star}}_t = \sum_{\tau=t}^{\infty} \gamma^{\tau-t} \varepsilon_{\tau} + \gamma \rho_{\F}(\tilde V^{\tilde \pi^{\star}}) \sum_{\tau=t}^{\infty} \gamma^{\tau-t} \delta_{\tau}$.

This shows that
\begin{align} \label{eq:value_approx_aux_bound}
    \lvert V^{\tilde \pi}_t(h_t) - \tilde V^{\tilde \pi^{\star}}(\sigma_t(h_t)) \rvert \leq \eta^{\tilde \pi^{\star}}_t.
\end{align}
Note that $\tilde V^{\tilde \pi^{\star}}(\sigma_t(h_t)) = \tilde V^{\star}(\sigma_t(h_t))$. Finally, we have the bound on policy approximation:
\begin{align*}
    \lvert V^{\star}_t(h_t) -  V^{\tilde \pi}_t(h_t) \rvert &\stackrel{(a)}\leq \lvert V^{\star}_t(h_t) - \tilde V^{\star}(\sigma_t(h_t)) \rvert + \lvert V^{\tilde \pi}_t(h_t) - \tilde V^{\tilde \pi^{\star}}(\sigma_t(h_t)) \rvert \\
    &\stackrel{(b)}\leq 2\sum_{\tau=t}^{\infty} \gamma^{\tau-t} \varepsilon_{\tau} + 2 \gamma \rho_{\F}(\tilde V^{\star}) \sum_{\tau=t}^{\infty} \gamma^{\tau-t} \delta_{\tau},
\end{align*}
where $(a)$ follows from the triangle inequality and $(b)$ follows from \eqref{eq:value_approx_main_bound_V} and \eqref{eq:value_approx_aux_bound}. This establishes the bound in the main theorem on policy approximation.

\section{Instance independent bounds}

In this section, we illustrate \emph{instance independent} bounds by upper bounding $\rho_{\F}(V^\star_\xi)$ in terms of properties of the transition and reward functions. These \emph{instance independent} bounds have been established in \cite{subramanian2022approximate}.

When the IPM considered is the total variation distance, i.e., $\F = \F_{\TV} \coloneqq \{f : \SPAN(f) \le 1\}$ (where $\SPAN(f)$ is the span semi-norm of a function), we have
\begin{align*}
    \rho_{\F}(V^\star_\xi) = \SPAN(V^\star_\xi) \leq \frac{\SPAN(r)}{(1-\gamma)}.
\end{align*}

Next, when the IPM considered is the Wasserstein distance, i.e., $\F = \F_{\Was} \coloneqq \{ f : \LIP(f) \le 1 \}$ (where $\ALPHABET X$ is a metric space and $\LIP(f)$ is the Lipschitz constant of the function $f$, computed with respect to the metric on $\ALPHABET X$); the Lipschitz constants $L_r$ of the reward function and $L_P$ of the state transition matrix are finite; and $\gamma L_P < 1$; then we have
\begin{align*}
    \rho_{\F}(V^\star_\xi) = \LIP(V^\star_\xi) \leq \frac{L_r}{(1-\gamma L_P)}.
\end{align*}
\section{Proof of claims in Theorem 2}

\subsection{\texorpdfstring{$W^2_t(z, a)$}{W\^2\_t} converges to 0 almost surely.}\label{app:W2}
The definition of the learning rate $\{\alpha_t\}_{t \ge 1}$ implies that 
   \begin{equation}
       W^2_{t+1}(z, a) =
       \lim_{t \to \infty}
       \dfrac{\dfrac{1}{t} \sum_{k=0}^{t-1} F^2_k(z, a) \IND_{\{Z_{k}=z, A_k=a\}}}
       {\dfrac{1}{t} \sum_{k=0}^{t-1} \IND_{\{Z_{k}=z, A_k=a\}}}
       \label{eq:W2}
   \end{equation}
   
    Recall that according to Strong Law of Large Numbers for for Markov chains, given a aperiodic and irreducible Markov chain $\{X_t\}_{t \ge 1}$, $X_t \in \ALPHABET X$, with stationary distribution $\rho$ and any (measurable) function $h$, the long run average is equal to the steady state value, i.e., 
    \begin{equation}\label{eq:SLLN}
        \lim_{T \to \infty} \frac{1}{T} \sum_{t = 1}^T h(X_t) = \sum_{x \in \ALPHABET X} h(x) \rho(x),
        a.s.. 
    \end{equation}
    
    Let $\xi(s,y,z,z_{+},a) = \xi(s,y,z,a) P_\xi(z_{+} \mid z,a)$. Then,
    the numerator of~\eqref{eq:W2} may be simplified as
    \begin{equation}\label{eq:num}
         \sum_{(s,y,z,z_{+},a)} \biggl[
         r(s,a) - r_\xi(z,a) 
         + \gamma V^\star_\xi(z_{+})
         - \gamma \sum_{z'} P_{\xi}(z' | z, a) V^\star_\xi(z')
        \biggr]\xi(s,y,z,z_{+},a).
   \end{equation}
    By definition of $r_\xi(z,a)$, we have 
    \[
        \sum_{(s,z,a)}[ r(s,a) - r_\xi(z,a) ] \xi(s,z,a)
        = \sum_{z,a} \Bigl[ \medop\sum_{s} r(s,a) \xi(s|z,a) - r_\xi(z,a) \Bigr] = 0
    \]
    Similarly, by definition of $\xi(s,y,z,z_+,a)$, we have 
    \[
         \sum_{s,y,z,z_+,a} \bigl[ V^\star_\xi(z_{+})
         - \medop\sum_{z'} P_{\xi}(z' | z, a) V^\star_\xi(z') \bigr]
         \xi(s,y,z,a) P_{\xi}(z_{+} \mid z, a)
         = 0
    \]
    Substituting the above two equations in~\eqref{eq:num}, we get that the numerator of~\eqref{eq:W2} is zero. Hence, $W^2_t(z,a) \to 0$, a.s., for all $(z,a)$.

\subsection{\texorpdfstring{$W^1_t(z,a)$}{W\^1\_t(z,a)} cannot remain above \texorpdfstring{$C\epsilon$}{Cepsilon} forever}

    We first state \cite[Lemma 3]{jaakkola1993convergence}.\footnote{In \cite{jaakkola1993convergence}, the result is stated with using the iteration 
    \(
            X_{t+1}(x) = (1 - \alpha(x)) X_{t}(x) + \gamma \beta_t(x) \| X_t \|.
    \)
    where $\gamma \in (0,1)$. But $\gamma$ plays no role in the proof argument of \cite[Lemma 3]{jaakkola1993convergence} and may be omitted.}
    \begin{lemma}\label{lem:jaakkola}
        Let $\ALPHABET X$ be a finite set. Consider a vector valued stochastic process $\{X_t\}_{t \ge 1}$, where
        \[
            X_{t+1}(x) = (1 - \alpha(x)) X_{t}(x) + \beta_t(x) \| X_t \|.
        \]
        where $\sum_{t \ge 1} \alpha_t(x) = \infty$, $\sum_{t \ge 1} \alpha_t^2(x) < \infty$, $\sum_{t \ge 1} \beta_t(x) = \infty$, $\sum_{t \ge 1} \beta^2_t(x) < \infty$, almost surely,  and $\EXP[ \beta_t(x) ] \le \EXP[ \alpha_t(x) ]$. 
        Then, $X_t(x)$ converges to zero almost surely for all $x \in \ALPHABET X$.  
    \end{lemma}
    
    Assume that for any $t > T(\omega,\epsilon)$, we have $\|W^1_t\|_{\infty} > C\epsilon$. Recall that we have shown  in~\eqref{eq:F1-bound} that if $\|W^1_t\|_{\infty} > C \epsilon$ then $F^1_t(z,a) \le \| W^1_t \|_{\infty}$. Define $\beta_t(z,a) = \alpha_t(z,a) F^1_t(z,a)/\|W^1_t\|_{\infty}$. Then, the recursion for $W^1_t$ can be written as
    \[
        W^1_{t+1}(z,a) = (1 - \alpha_t(z,a)) W^1_t(z,a) + \beta_t(z,a) \|W^1_t\|_{\infty}
    \]
    which is an instance of the linear iteration described in \autoref{lem:jaakkola}. Therefore, by \autoref{lem:jaakkola}, $W^1_t(z,a)$ converges to zero almost surely. But this contradicts the assumption that $\|W_t\|_{\infty} > C\epsilon$. Hence, the assumption cannot be true.

\subsection{If \texorpdfstring{$\|W^1_t\|_{\infty}$}{norm of W\^1\_t} hits below \texorpdfstring{$C\epsilon$}{C epsilon} then it stays there}

    Recall that we have shown in \eqref{eq:F0-bound} that $F^1_t(z,a) \le \gamma \| \Delta_t \|_{\infty}$. Now suppose that for some $t > T(\epsilon,\omega)$, $\|W^1_t \|_{\infty} < C\epsilon$. Then, 
    \[
       F^1_t(z,a) \le \gamma (\| W^1_t \|_{\infty} + \|W^2_t\|_\infty) \le \gamma (C\epsilon + \epsilon) 
       \le \gamma (1 + C) \epsilon \le C \epsilon
    \]
    where we have used the fact that $\gamma (1 + \frac 1C) < 1$ in the last inequality. 
    Thus,
    \[
        |W^1_{t+1}(z,a) \le (1 - \alpha_t(z,a))| W^1_t(z,a)| + \alpha_t(z,a) |F^1_t(z,a)|
        \le (1 - \alpha_t(z,a)) C\epsilon + \alpha_t(z,a) C\epsilon 
        = C\epsilon.
    \]
    Hence $\|W_{t+1}\|_{\infty} \le C\epsilon$.

\section{Recurrent Q-learning with AIS}

The training phase of RQL-AIS consists of sampling batches of sequential data from the R2D2 replay buffer and training: 1) the AIS generator components using the AIS loss, 2) the Q-function using the multi-step Q-learning loss. The full RQL-AIS algorithm is presented in \autoref{alg:r2d2-AIS}. In this section, we assume the AIS history generator $\sigma(h,z_0)$ is an RNN function which receives the initial hidden state $z_0$ and input history $h$ and outputs $z$ which is the hidden state corresponding to the sample at the end of the history sequence $h$.

\begin{algorithm}
\caption{Recurrent Q-learning with AIS (RQL-AIS)}\label{alg:r2d2-AIS}
\begin{algorithmic}[1]
	\STATE init $\sigma$, $\hat{P^{y}}$, $\hat{r}$, $Q_{\theta}$ to random networks, init $\theta' \leftarrow \theta$ and R2D2 buffer $\mathbf{D} \leftarrow \{\}$, L is the R2D2 sequence length, $\phi$ is the pretrained observation encoder
      \FOR{episode $=1, M$}
        \STATE start episode, init history ${h} \leftarrow \{\}$
        \WHILE{not done}
            \STATE receive observation $y_{t}$, encode $y_{t}$ with $\phi$ and append to history h.
            
            \item[]  \texttt{\textcolor{gray}{/* \textit{$\epsilon$-greedy policy for data collection} */}}
            \STATE Select action $a_{t} = \operatorname*{argmax}_{a}  Q_{\theta}(\sigma(h,\mathbf{0}),a)$ with probability $1 - \epsilon$ otherwise random action with probability $\epsilon$, append action to h
            \item[]  \texttt{\textcolor{gray}{/* \textit{Add data to buffer in regular intervals (Once every L steps) } */}}
            \STATE Add sequence to $\mathbf{D}$
            
		    \STATE Sample batch of experience sequences $(h_b,z,y_{1:L}, r_{1:L}, a_{1:L})$ from $\mathbf{D}$
                \item[]  \texttt{\textcolor{gray}{/* \textit{Q-function and AIS component training} */}}
                \STATE Initialize $\sigma$ with z and burn-in history $h_b$ and detach gradients: $z^{\prime} = \sigma(h_b,z)$
                \STATE Compute AIS loss for $(z^{\prime},y_{1:L}, r_{1:L}, a_{1:L})$ and update $\sigma$, $\hat{P^{y}}$ and $\hat{r}$ 
		    \STATE Compute n-step Q-learning loss for $(z^{\prime},y_{1:L}, r_{1:L}, a_{1:L})$ and update $Q_{\theta}$		    
            \STATE Update target Q-function in regular intervals $\theta' \leftarrow \theta$
		\ENDWHILE
	  \ENDFOR
    \end{algorithmic}
\end{algorithm}

\subsection{Recurrent replay buffer}

With non-recurrent DQN for fully-observable environments, individual single step samples $(s,a,r,s^{\prime})$ are saved in the replay buffer. With Recurrent Q-learning, full histories are required with each sample to allow for the accurate computation of the recurrent hidden states. The R2D2 method \cite{kapturowski2018recurrent} suggests saving fixed length sequences of observation, action and rewards $(y,a,r)$ in the replay buffer. Furthermore, to recreate the internal state of the RNN during training, it is suggested to store the following with each sample sequence: 1) A truncated preceding "Burn-In" history, 2) An initial RNN internal hidden state representing the internal RNN state at the start of the "Burn-In" history. The "Burn-In" sequence allows us to unroll the recurrent network on a truncated portion of preceding history with each sample sequence. The "Burn-In" portion is only used for unrolling the recurrent network and updates are only done on the main sequence. Additionally, the recurrent state at the start of this "Burn-In" history is also saved with each sample in the buffer. Both of these allow the recurrent network to recreate the hidden state at the start of the main sample sequence as accurately as possible. We utilize the same replay buffer for RQL-AIS and our ND-R2D2 baseline.

The problem of inaccuracy in recreated RNN hidden states is called representation drift and the effect of distributed training on this problem is investigated in the original R2D2 paper \cite{kapturowski2018recurrent}. The representation drift problem is less severe when the data samples are more recent as their saved internal states comes from a version of the RNN which is closer to the most recently updated version. It is demonstrated that the distributed setting can mitigate this issue with higher number of actors \cite{kapturowski2018recurrent}. The higher number of actors allows for a larger amount of generated experience, which results in data being added to the buffer more frequently, meaning older samples are less likely to be sampled by the learner unit, effectively alleviating the isssue. In our setup, distributed training is both unnecessary and infeasible. The environments used for our experimental setup require significantly less data to be successfully solved and also the computational resources available at our disposal do not allow for a distributed R2D2 implementation.  In order to mitigate the representation drift problem in our setup, we choose a different set of R2D2 hyperparameters compared with \cite{kapturowski2018recurrent}. First, we choose a significantly smaller main sequence length for our R2D2 buffer and second, we choose a much longer "Burn-In" sequence length compared with the original R2D2 implementation. This combination allows us to achieve a good performance with the R2D2 replay buffer while using a single actor setup. By using a shorter main sequence length and a larger batch size, single step samples are less correlated in each batch which helps with the performance of Q-learning. Furthermore, adjacent saved samples are no longer overlapping each other as they do in the original R2D2 implementation as we find that unnecessary considering the shorter main sequence length.

\subsection{Q-learning}
In the main text, we discussed the AIS loss used for training the AIS-generator. Assuming $Z_{t}$ is the recreated AIS state representation for each single-step sample, we train the Q-function using the n-step Q-learning loss:

\begin{equation}
 \left| \sum_{k=0}^{n-1} R_{t+k} \gamma^{k}+\gamma^{n} Q_{\theta^{-}}\left(Z_{t+n}, a^{*}\right) - Q_{\theta}\left(Z_t, A_t\right) \right|^{2} , \quad a^{*}=\underset{a}{\arg \max } Q_{\theta}\left(Z_{t+n}, a\right).
\end{equation}

with $Q_{\theta}$ being the Q-function and $Q_{\theta^{-}}$ being the target Q-function. Also, $a^{*}$ is chosen by taking the arg-max of $Q_{\theta}(Z_{t+n}, a)$ similar to the double Q-learning approach \cite{https://doi.org/10.48550/arxiv.1509.06461}. The multi-step length $n$ is set to 5 for all experiments. Following the DQN implementation, target Q-function parameters $\theta^{-}$ are periodically updated with the most recent value of $\theta$. Gradients from the Q-learning loss do not backpropagate through the AIS generator parameters.

We can further utilize prioritized experience replay (PER) \cite{https://doi.org/10.48550/arxiv.1511.05952} with RQL-AIS. With prioritized sampling, sequences in the R2D2 buffer are sampled with non-uniform probabilities. Each sequence's priority is the absolute value of the n-step Q-learning error averaged over all samples in the sequence. Similar to the original PER method \cite{https://doi.org/10.48550/arxiv.1511.05952}, we use importance-sampling (IS) weights to correct the change in the data distribution introduced by the non-uniform weights. The importance sampling weights are applied to rescale both the AIS loss and the Q-learning loss.

\subsection{Environments}

\begin{figure}[!t]
\includegraphics{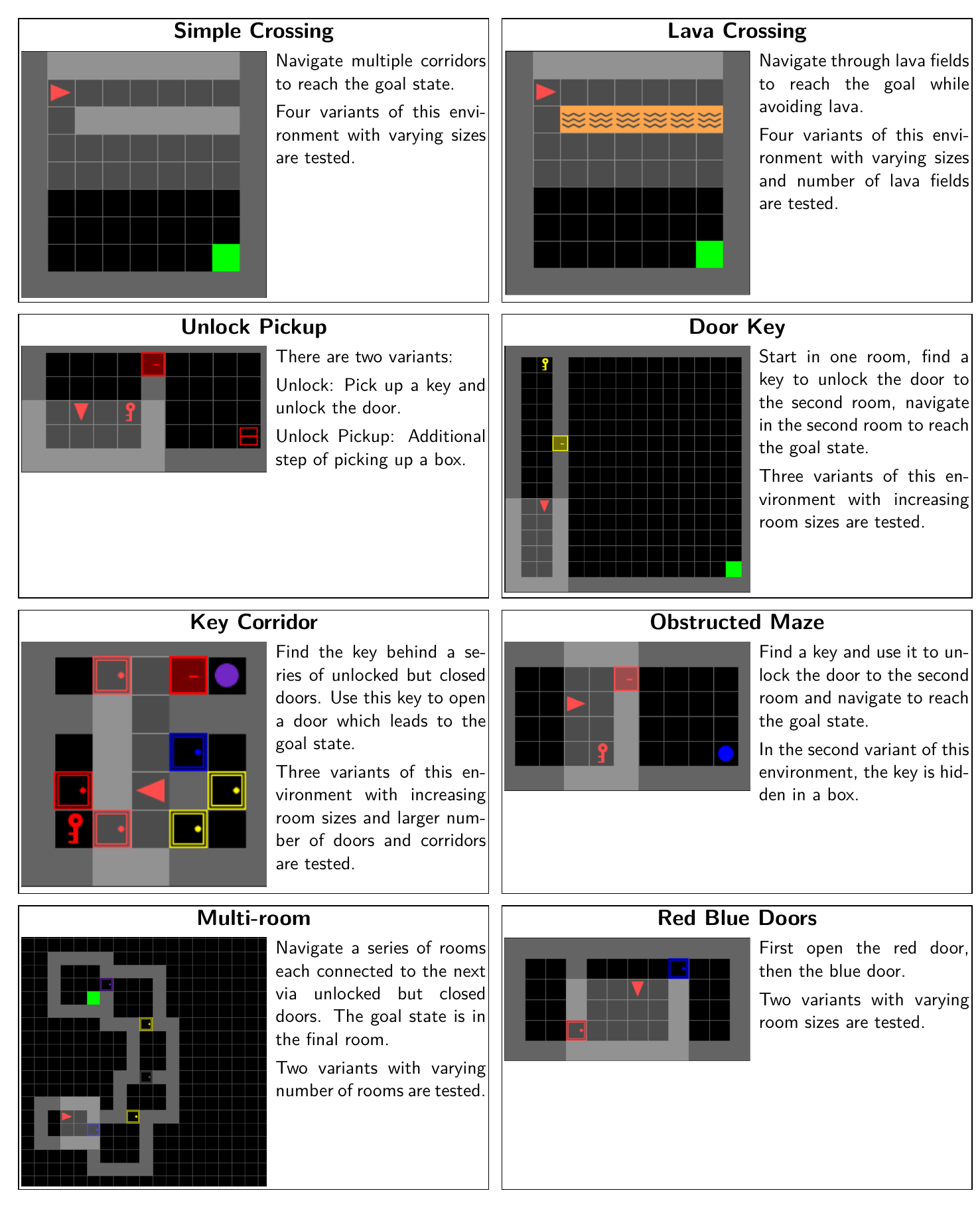}
\caption{Details of the different MiniGrid environments}
\label{fig:envs}
\end{figure}

We evaluate the performance of RQL-AIS and other baseline algorithms on 22 environments form the Minigrid testbed~\cite{minigrid}. All environments are partially observed grid worlds where the agent has to solve a series of tasks. At each time, the agent observes a $7 \times 7$ view in front of it; each observed tile is encoded as a three-dimensional tuple of (\textsc{object\_id}, \textsc{color\_idx}, \textsc{state}). 
The agent cannot see through walls and closed doors nor see inside closed boxes. The agent receives a positive reward of $+1$ on successful completion of the task and no intermediate rewards. The environments are grouped into eight major categories, as shown in \autoref{fig:envs}

Given the large (discrete) observation space, we follow the methodology of~\cite{subramanian2022approximate} and use a pre-trained auto-encoder to compress the observations. The auto-encoder is trained on a dataset of random agent observations, separately for each environment. During the RL phase, the auto-encoder is frozen and doesn't change.

\subsection{Algorithms tested}

We evaluate the performance of the algorithms described below. Each algorithm has two variants: one using uniform sampling (denoted by ``U'') and another using prioritized experience replay (denoted by ``PER'')
\begin{itemize}
    \item \texttt{RQL-AIS-U} and \texttt{RQL-AIS-PER}: Recurrent Q-learning algorithm with AIS representations. 
    \item \texttt{ND-R2D2-U} and \texttt{ND-R2D2-PER}: Non-distributed R2D2 
    \item \texttt{ND-R2D2-AE-U} and \texttt{ND-R2D2-AE-PER}:
    Non-distributed R2D2 where observations are compressed using pre-trained auto-encuders.
\end{itemize}
Note that both \texttt{RQL-AIS} and \texttt{ND-R2D2-AE} use the same pretrained auto-encoders.

We train all algorithms for $T = 4 \cdot 10^6$ environment steps for $N = 5$ seeds. The agents are trained using epsilon-greedy approach~\cite{mnih2013atari} to allow for exploration. The value of epsilon is chosen to smoothly decay between $\epsilon_{\text{start}}$ and $\epsilon_{\text{end}}$ according to:
\[
    \epsilon_{\text{expl},t} = \epsilon_{\text{end}} + (\epsilon_{\text{start}} - \epsilon_{\text{end}}) \exp\left(- \frac{t}{\epsilon_{\text{decay}}}\right).
\]

The trained agents are evaluated at regular intervals. During evaluation, the agents are tested for 10 episodes and actions are chosen greedily with respect to the trained Q-function. 

\subsection{Training results}

The mean and the standard deviations of the final performance for all algorithms is shown in \autoref{tab:my_label--}. \autoref{fig:performance-full} shows the training curves for all 22 tested environments.

Prioritized experience replay can offer a significant boost to both algorithms in more difficult environments and we see RQL-AIS-PER to show the best performance among all tested algorithms. This variant of RQL-AIS is capable of solving all of the tested environments. We observe that variants of RQL-AIS and ND-R2D2 with prioritized sampling generally perform worse in simpler environments. Prioritized sampling relies on priority weights obtained from the absolute value of the n-step Q-learning loss. In the initial phases of training, Q value estimates are inaccurate, making prioritized sampling hurt performance initially while the uniform sampling variants of these algorithms are able to quickly solve these environments. We observe that AIS state representations are very effective in solving the more difficult environments and RQL-AIS to show superior performance in more difficult environments regardless of the sampling approach used for the replay buffer.

We are reporting two additional variants of R2D2 which work with the pretrained encoders (similar to RQL-AIS). The main ND-R2D2 baseline used in the main text does not use pretrained encoders as we believe pretrained encoders are not an essential component of ND-R2D2. Nevertheless, we report results for ND-R2D2 using these pretrained encoders so that their effect can be decoupled from the other factors. We observe that their performance is very similar to ND-R2D2 working on raw observations. 

\begin{table}[!tbp]
    \centering
    \includegraphics[width=\textwidth]{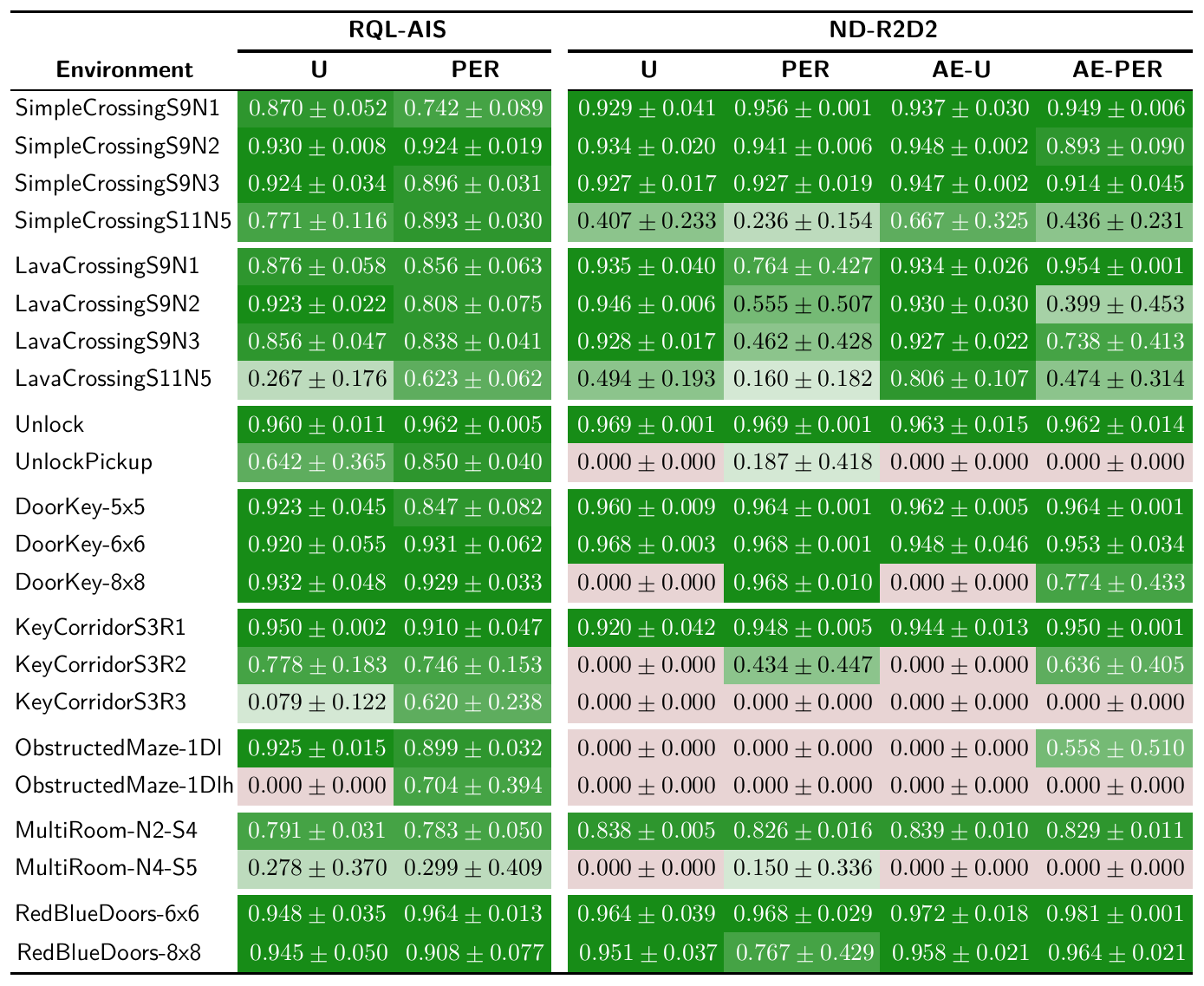}
    \caption{Comparison of RQL-AIS with the different variants of ND-R2D2 on the MiniGrid benchmark.}
    \label{tab:my_label--}
\end{table}

\begin{figure}[!tbp] 
\centering
\includegraphics[width=\textwidth]{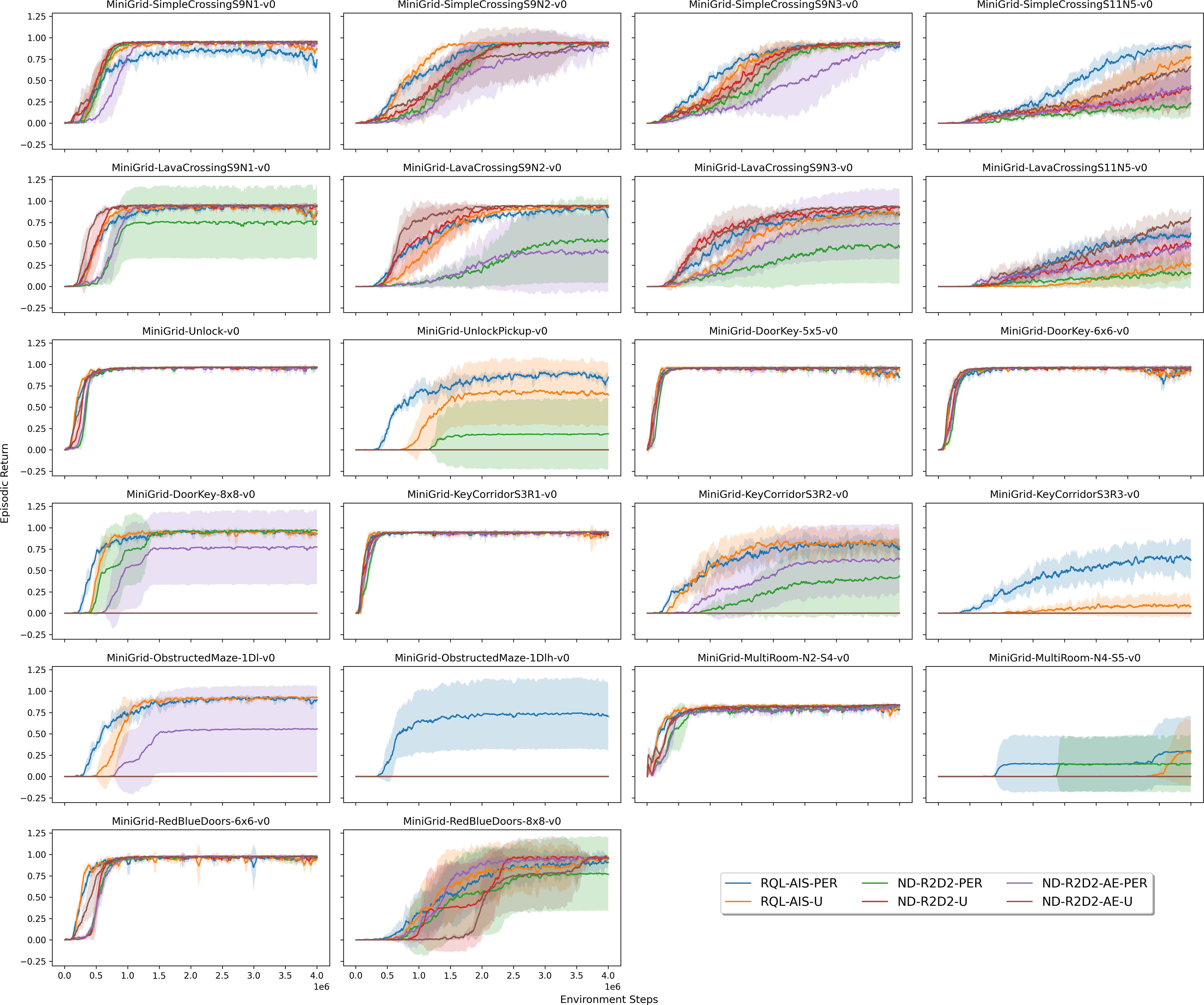}
\caption{Training curves from all 22 tested MiniGrid environments. Results for both uniform sampling and prioritized sampling variants of RQL-AIS and ND-R2D2 are provided. Two additional variants of R2D2 are also included which use the pretrained encoders.}
\label{fig:performance-full}
\end{figure}

\section{Effect of the hyperparameter \texorpdfstring{$\lambda$}{lambda}}

In this section, we report the effect of the $\lambda$ hyperparameter on performance for four candidate environments. The plots show the performance of RQL-AIS with uniform sampling over time. We observe that the performance of RQL-AIS is robust with respect to the choice of the $\lambda$ hyperparameter with minor changes in performance seen between the different variants. We only see the variant with $\lambda = 0.999$ to perform slightly worse than the other variants. Very high values of $\lambda$ can diminish the gradients coming from the observation prediction component and hurt performance. We believe that the observation prediction component of the AIS loss has a big impact on the performance of RQL-AIS. This is understandable as MiniGrid environments are sparse reward and zero rewards are received on every environment interaction except the terminal steps leading to a successful completion of the task. Observation prediction can provide a much richer learning signal allowing the agent to learn more meaningful AIS state representations. The performance of RQL-AIS with different values of $\lambda$ is reported in \autoref{fig:performance-lambda}. We choose $\lambda = 0.5$ for all experiments involving RQL-AIS.

\begin{figure}
\centering
\includegraphics[width=\textwidth]{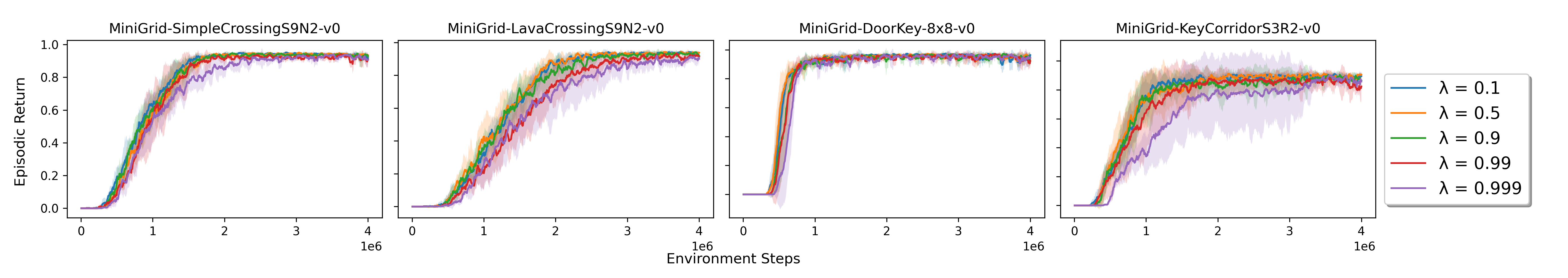}
\caption{Performance plots of RQL-AIS using different values of the $\lambda$ hyperparameter.}
\label{fig:performance-lambda}
\end{figure}

\section{Implementation details}

In this section, we first present the hyperparameters used for implementing RQL-AIS and ND-R2D2. These hyperparameters are shared between the two algorithms (unless specifically stated). 
We report the hyperparameter values in Table \ref{hyp}.

\begin{table}
    \centering
\begin{tabular}{@{}cc@{}}
 \toprule
Number of Seeds & 5 \\
 Number of Environment steps  & $4 \times 10^6$  \\
 Discount Factor $\gamma$  & $0.99$  \\
 $\epsilon_{\text{start}}$ & 1.0 \\
 $\epsilon_{\text{end}}$ & 0.05 \\
$\epsilon_{\text{decay}}$ & 400000 \\
 Evaluation Interval & 5000 environment steps \\
 R2D2 sequence length &   10\\
 R2D2 Burn-In sequence length &  50 \\
 Optimizer  &  Adam \cite{https://doi.org/10.48550/arxiv.1412.6980} \\[3pt]
 AIS $\lambda$ (for RQL-AIS) &   0.5 \\
 AIS learning rate (for RQL-AIS) &   $10^{-3}$ \\
 Q-function learning rate (for RQL-AIS) &   $10^{-3}$\\
 learning rate (for ND-R2D2) & $10^{-3}$\\
 Minibatch Size &   256 \\
 Network Update Interval &  10 environment steps\\[3pt]
 Target network update interval & 100 updates to main network  \\
 Priority Exponent (PER $\alpha$) &  $0.6$ \\
 Importance Sampling Exponent (PER $\beta$) &  $[0.4,1.0]$ \\
 LSTM hidden size ($d_{Z}$) &  $128$ \\
 \bottomrule
\end{tabular}
\caption{Hyperparameter values.}
\label{hyp}
\end{table}

In all MiniGrid environments, the agent receives $7 \times 7 \times 3$ sized observation vectors. In order to compress the observations into a more compact representation, we are using pretrained encoders. These are simple auto-encoders \cite{Kramer1991NonlinearPC} that are trained on datasets of random agent observations. The encoders are trained for 100 epochs on this dataset. We use MLP layers with ReLU nonlinearities \cite{Fukushima1975} for the encoder and the decoder architecture. The encoder weights are frozen during the RL phase. The encoder and decoder architectures are as follows:

$$
\begin{array}{cccc}
\hline \text{Encoder} & \text{Decoder} \\
\hline \operatorname{Linear}\left(147 , 96\right) & 
\operatorname{Linear}\left(64 , 96\right) \\
\operatorname{ReLU} & \operatorname{ReLU} \\
\operatorname{Linear}\left(96 , 64\right) & \operatorname{Linear}\left(96 , 147\right) \\
&  \operatorname{Tanh} \\
\hline
\end{array}
$$

RQL-AIS has the following four components: the recurrent history compression function $\hat{\sigma}$, the reward prediction function $\hat{r}$, the observation prediction function $\hat{P}^y$ and the action-value function $\hat{Q}_{\theta}$. During our experimentation, we use Exponential Linear Units (ELUs) \cite{https://doi.org/10.48550/arxiv.1511.07289} for the nonlinear layers. For the recurrent component, we are using an LSTM \cite{article} function. Linear layers with input sizes $n$ and output sizes $m$ are denoted by Linear$(n,m)$ and an LSTM cell with input of size $n$ and a hidden vector size of $m$ is denoted by LSTM$(n,m)$. $n_O$ and $n_A$ denote the observation vector size and the action space size. The neural network architectures of the four components are:

$$
\begin{array}{cccc}
\hline \hat{\sigma} & \hat{r} & \hat{Q}_{\theta} & \hat{P}^y \\
\hline \operatorname{Linear}\left(n_O+n_A, d_{\hat{Z}}\right) & \operatorname{Linear}\left(n_A+d_{\hat{Z}}, \frac{1}{2} d_{\hat{Z}}\right) & \operatorname{Linear}\left(d_{\hat{Z}}, d_{\hat{Z}}\right) &
\operatorname{Linear}\left(n_A+d_{\hat{Z}}, \frac{1}{2} d_{\hat{Z}}\right) \\
\operatorname{ELU} & \operatorname{ELU} & \operatorname{ELU} & \operatorname{ELU} \\
\operatorname{LSTM}\left(d_{\hat{Z}}, d_{\hat{Z}}\right) & \operatorname{Linear}\left(\frac{1}{2} d_{\hat{Z}}, 1\right) & \operatorname{Linear}\left(d_{\hat{Z}}, d_{\hat{Z}}\right) &
\operatorname{Linear}\left(\frac{1}{2} d_{\hat{Z}}, n_O\right) \\
& & \operatorname{ELU} \\
& & \operatorname{Linear}\left(d_{\hat{Z}},n_A\right) \\
\hline
\end{array}
$$

For the experiments with ND-R2D2 we have divided the architecture into two components: the recurrent unit $\hat{\sigma}$ and the Q-function $\hat{Q}_{\theta}$. This is done for simplicity of implementation and consistency with the RQL-AIS notation but both are trained with the Q-learning loss. The Q-learning variant which does not use the pretrained autoencoders has an extra layer in its recurrent unit to allow for extra processing of the higher-dimensional inputs.

$$
\begin{array}{cccc}
\hline \hat{\sigma} (AE) & \hat{\sigma} (\text{raw}) & \hat{Q}_{\theta} \\
\hline \operatorname{Linear}\left(n_O+n_A, d_{\hat{Z}}\right) & 
\operatorname{Linear}\left(n_O+n_A, d_{\hat{Z}}\right) &
\operatorname{Linear}\left(d_{\hat{Z}}, d_{\hat{Z}}\right) \\
\operatorname{ELU} & \operatorname{ELU} & \operatorname{ELU} \\
\operatorname{LSTM}\left(d_{\hat{Z}}, d_{\hat{Z}}\right) & \operatorname{Linear}\left(d_{\hat{Z}}, d_{\hat{Z}}\right) & \operatorname{Linear}\left(d_{\hat{Z}}, d_{\hat{Z}}\right) \\
&  \operatorname{ELU} & \operatorname{ELU} \\
&  \operatorname{LSTM}\left(d_{\hat{Z}}, d_{\hat{Z}}\right) & \operatorname{Linear}\left(d_{\hat{Z}},n_A\right)  \\
\hline
\end{array}
$$

\section{Limitations}

In this work, we have conducted an extensive theoretical analysis into the convergence of recurrent Q-learning. Our analysis is restricted to a setting where the true state space and the recurrent state space are both finite and therefore the tabular setup can be used to represent the Q-function. In practice, the state space is either continuous or large so that some form of function approximation is needed and the recurrent state is continuous. 

In addition, we assume that the AIS-generator functions including the recurrent history compression function are fixed throughout the Q-learning iterations. In a practical algorithm such as \texttt{RQL-AIS}, the representation is learnt in parallel with the Q-function.

Our results are derived under a slightly strong assumption on the learning rates \textbf{(A3)}, while most convergence results (for MDPs) use the weaker assumption \textbf{(A3')}. 

In our empirical evaluation, we train the representation using an observation prediction as a proxy for the AIS-predictor. It is shown in \cite[Proposition 8]{subramanian2022approximate} that having a good observation predictor along with a recursively updateable AIS is a sufficient condition for having a good AIS-predictor. However, predicting all the observations might not be necessary and using an AIS-predictor may, in principle, need a smaller representation.

\end{document}